\documentclass{article}

\usepackage{microtype}
\usepackage{graphicx}
\usepackage{subfigure,overpic}
\usepackage{booktabs,bm} 

\usepackage{hyperref}

\usepackage[accepted]{icml2020_arxiv}

\usepackage[cmex10]{amsmath}
\usepackage{amsfonts}
\usepackage{amsthm}
\usepackage{amssymb}
\usepackage{mathrsfs}
\usepackage{latexsym}
\usepackage{cases}
\usepackage{comment}

\DeclareMathOperator*{\argmin}{arg\,min}
\newtheorem{theorem}{Theorem}
\newtheorem{lemma}{Lemma}
\newtheorem{definition}{Definition}

\newtheorem{proposition}{Proposition}

\newcommand{\bbP}{{\mathbb P}}

\newcommand{\bbE}{{\mathbb E}}

\newcommand{\cB}{{\cal B}}
\newcommand{\cP}{{\cal P}}
\newcommand{\cS}{{\cal S}}
\newcommand{\cX}{{\cal X}}
\newcommand{\cY}{{\cal Y}}

\newcommand{\cG}{{\cal G}}
\newcommand{\cV}{{\cal V}}
\newcommand{\cE}{{\cal E}}
\newcommand{\cT}{{\cal T}}
\newcommand{\cL}{{\cal L}}
\newcommand{\cM}{{\cal M}}

\newcommand{\brx}{{\boldsymbol{\mathrm{x}}}}
\newcommand{\bry}{{\boldsymbol{\mathrm{y}}}}
\newcommand{\brX}{{\boldsymbol{\mathrm{X}}}}
\newcommand{\brY}{{\boldsymbol{\mathrm{Y}}}}
\newcommand{\nbd}{{\mathrm{nbd}}}
\newcommand{\pa}{{\mathrm{pa}}}
\newcommand{\CL}{{\mathrm{CL}}}

\newcommand{\atanh}{{\mathrm{atanh}}}
\newcommand{\T}{{\mathrm{T}}}
\newcommand{\KA}{{\mathrm{KA}}}

\newcommand{\SGA}{{\mathrm{SGA}}}

\newcommand{\sS}{{\mathscr S}}

\newcommand{\sE}{{\mathscr E}}

\newcommand{\floor}[1]{\left\lfloor #1\right\rfloor}

\newcommand*{\medcup}{\mathbin{\scalebox{1.5}{\ensuremath{\cup}}}}%
\newcommand*{\meduplus}{\mathbin{\scalebox{1.4}{\ensuremath{\uplus}}}}%

\newcommand{\myrule}{\,\rule[2.5pt]{3mm}{0.5pt}\,}

\usepackage{enumitem}%
\setlist[itemize]{noitemsep, topsep=0pt}
\setlist[enumerate]{noitemsep, topsep=0pt}

\icmltitlerunning{SGA: A Robust Algorithm for  Partial Recovery of Tree-Structured  Graphical  Models with Noisy Samples}

\begin{document}

\twocolumn[
\icmltitle{SGA: A Robust Algorithm for  Partial Recovery of Tree-Structured  Graphical  Models with Noisy Samples}



%
\begin{icmlauthorlist}
\icmlauthor{Anshoo Tandon}{nus1}
\icmlauthor{Aldric H.~J.~Yuan}{nus2}
\icmlauthor{Vincent Y.~F.~Tan}{nus1,nus2}
\end{icmlauthorlist}

\icmlaffiliation{nus1}{Department of Department of Electrical \& Computer Engineering, National University of Singapore, Singapore}
\icmlaffiliation{nus2}{Department of Mathematics, National University of Singapore, Singapore}
\icmlcorrespondingauthor{Anshoo Tandon}{anshoo.tandon@gmail.com}

\icmlkeywords{Ising model, non-identical noise, partial tree recovery, impossibility result}

\vskip 0.3in
]



\printAffiliationsAndNotice{}  

\begin{abstract}
We consider   learning    Ising tree  models when the observations from the nodes are corrupted by independent but non-identically distributed noise with unknown statistics. \citet{Katiyar20arxiv} showed that although the {\em exact} tree structure cannot be recovered,    one can recover a {\em partial} tree structure; that is, a structure belonging to the equivalence class containing the true tree. This paper presents a systematic improvement of~\citet{Katiyar20arxiv}. First, we present a novel impossibility result by deriving a bound on the necessary number of samples for partial recovery. Second, we derive a significantly improved sample complexity result in which the dependence on the minimum correlation $\rho_{\min}$ is $\rho_{\min}^{-8}$ instead of $\rho_{\min}^{-24}$. 
Finally, we propose Symmetrized Geometric Averaging (SGA),  a more statistically robust algorithm for partial tree recovery. We provide error exponent analyses and extensive numerical results on a variety of trees  to show that the sample complexity of SGA is significantly better than the algorithm of~\citet{Katiyar20arxiv}.  SGA can be  readily extended to Gaussian models and is shown via numerical experiments to be  similarly superior.
	
	
\end{abstract}

\section{Introduction} \label{sec:Intro}
Graphical models provide a succinct diagrammatic representation of the dependencies among a set of random variables. The vertices of the graph are in one-to-one correspondence with the random variables while the edges encode conditional independence relationships. 
Graphical models have found applications in domains from biology~\cite{Friedman2004}, to coding theory~\cite{Kschischang98}, social networks~\cite{LauritzenBook} and computer vision~\cite{Besag86}. For a   detailed description of graphical models, the reader is referred to~\citet{Wainwright08_FnT}.

This paper concerns the learning of the structure of a certain class of graphical models from data~\cite{Johnson2007AISTATS}. In particular, we consider undirected Ising   graphical models that are Markov on trees. While this is a classical problem that has been studied extensively~\cite{ChowLiu68,BreslerKarzand18}, here we focus on a relatively unexplored problem in which the vector-valued samples that are presented to the learner are {\em corrupted in independent but non-identically distributed noise}. This scenario is motivated by (at least) two different real-world scenarios. Firstly, suppose each component is obtained at spatially distributed locations of a large sensor network. The scalar-valued components need to be transmitted to a fusion center for the reconstruction of the graphical structure of the data. Due to the presence of disturbances (e.g., pollution, ambient noise), the components will inevitably be corrupted and the corruption levels are different as the distances of the spatially distributed devices to the fusion center are different. Secondly, imagine that a drug company would like to understand the inter-dependencies between different chemicals to design an effective vaccine given training samples obtained from human subjects. To thwart the company’s attempts, a competitor corrupts certain measurements or features of the human subjects \cite{WangGu2017ICML}, where these corruptions may be non-identically distributed. In these scenarios, we would like to design robust algorithms to recover, as best as possible, a certain structure that is ``close’’ to the true one. 

Motivated by these examples, we consider a tree learning problem in which each component of the vector-valued samples may be corrupted in a non-identical but independent (across components) manner. Because the corruption noises are non-identical, the ordering of the observed correlations are distorted, and hence the tree, in general, cannot be identified using the maximum likelihood Chow-Liu algroithm~\cite{ChowLiu68}. More precisely \citet{Katiyar20arxiv} showed that in such a case, even in the infinite sample limit, the structure can only be identified up to its equivalent class, a notion that will be defined precisely in the sequel. Building on their previous work for the robust learning of Gaussian graphical models~\cite{Katiyar19ICML}, \citet{Katiyar20arxiv} proposed a certain algorithm for the learning of Ising tree-structured models in which the samples are observed in noise. They provided a sample complexity result to recover the true tree up to its equivalence class. They showed that in this adversarial scenario, the go-to algorithm for learning tree-structure models -- the Chow-Liu algorithm~\cite{ChowLiu68} -- fails miserably but their algorithm is able to learn the model up to its equivalence class if the number of samples is sufficiently large. 

We significantly improve on   theoretical and algorithmic results in \citet{Katiyar20arxiv} and \citet{Nikolakakis19AISTATS}.

\vspace{-0.1in}

\begin{itemize}
	\item Firstly, we provide an information-theoretic impossibility result for tree structure recovery under non-identically distributed noise, that extends the result by~\citet{Nikolakakis19AISTATS} which assumes i.i.d. noise. 	Our impossibility result, which involves the construction and analysis of sufficiently many ``nearby’’ trees, elucidates the effect of noise on the sample complexity even as the number of samples and the number of nodes grow without bound; this desirable feature is not present in previous works on learning graphical models with noisy samples \citet{Nikolakakis19AISTATS}. 
	
	\item Secondly, by a careful analysis of the error events in the algorithm proposed by \citet{Katiyar20arxiv}, we significantly improve on the sample complexity result contained therein. In particular, the dependence of the sample complexity on the minimum correlation along the tree edges, $\rho_{\min}$, is improved from $\rho_{\min}^{-24}$ to $\rho_{\min}^{-8}$. 
	
	\item Finally, and most importantly, we propose an improvement to the IS\_NON\_STAR subroutine in \citet{Katiyar20arxiv}; we call our subroutine {\em Symmetrized Geometric Averaging} (SGA). This improvement is motivated by symmetry considerations in the error events in distinguishing a star versus non-star structure and the folklore theorem that ``more averaging generally helps’’. Indeed, we show through error exponent analyses~\cite{Tan11IT} using Sanov’s theorem~\citep{CoverBook06} that the error exponents of     SGA are higher (and hence better) than that in \citet{Katiyar20arxiv} for all but a small fraction of trees. This is   corroborated by extensive numerical experiments for a variety of trees of different structures, correlations, and corruption noises.  SGA is also shown to be amenable to improving the structure learning of {\em Gaussian} in addition to Ising trees.
\end{itemize}

\vspace{-0.1in}
\subsection{Related Work}
\vspace{-0.05in}
The learning of tree-structured graphical models dates back to the seminal work of \citet{ChowLiu68} who showed that the maximum likelihood estimation of the tree structure is equivalent to that of a maximum weight spanning tree problem with edge weights given by the empirical mutual information. \citet{ChowWagner73} showed that structure learning is consistent and \citet{Tan11IT} derived the error exponent. The estimates of the error probability of learning the tree structure was further refined by \citet{TandonTZ20_JSAIT}. \citet{BreslerKarzand18} considered a variant of the problem in which a tree was learned to make predictions instead of the traditional objective of inferring the structure. 

\citet{TandonTZ20_JSAIT} showed that the Chow-Liu algorithm is error exponent-optimal if the noises that corrupt the observations are independent and identically distributed. 
\citet{Nikolakakis19NonParametric} considered the learning of tree structures when the noise is possibly non-identically distributed, and derived conditions under which structure learning can be achieved using the Chow-Liu algorithm. We note that this setting is in contrast to recent work on robust tree learning under {\em adversarial} noise \cite{Cheng18NeurIPS}; in our work, {\em random} noise is added to clean samples.

\citet{Katiyar19ICML} showed that it is, in general, not possible to learn the exact tree structure for Gaussian graphical models when one is given independent but non-identically distributed noisy samples. This work was followed by \citet{Katiyar20arxiv}, who considered Ising tree models and derived an algorithm based on so-called {\em proximal sets} to learn the equivalence class (or cluster) of the true tree. 

Finally, we remark that there is a large body of literature on learning latent tree models (e.g., \citet{Choi11,parikh11}) in which one observes a subset of nodes from a tree. The marginal distribution of those observed nodes is, in general, not a tree. However, our work differs from learning latent trees, as we consider noisy observations from {\em all   nodes} of the true tree; this is not the case for latent trees.

\vspace{-0.1in}
\section{Preliminaries  and Problem Statement} \label{sec:prelim}
\vspace{-0.05in}
An {\em undirected graphical model} is a multivariate  probability distribution that factorizes according to the structure an  undirected graph~\cite{LauritzenBook}. Specifically, a $d$-dimensional random vector $\brX \triangleq (X_1, \ldots, X_d)$ is said to be \emph{Markov} on  $\cG = (\cV, \cE)$ with vertex (or node) set $\cV = \{1,\ldots,d\}$ and edge set $\cE \subset \binom{\cV}{2}$ if its distribution satisfies the (local) Markov property $P(x_i | x_{\cV \setminus i}) = P(x_i | x_{\nbd(i)})$ where $\nbd(i) := \{j \in \cV : \{i,j\} \in \cE\}$ is the {\em neighborhood} of node $i$. We focus on tree-structured graphical models $P$, where the underlying graph of $P$ is an acyclic and connected (tree) graph, denoted by $\T=\T_P = (\cV, \cE_{P})$ with $|\cV| = d$ and $|\cE_{P}| = d-1$. For an undirected tree, we may assume, without loss of generality, that node $1$ is the \emph{root} node  and we arrange all the nodes at different levels on a plane, with node $1$ at level-$0$. Then, the tree-structured graphical model $P$ can be alternatively factored as~\cite{ChowLiu68}
\begin{equation}
	P(\brx) = P_1(x_1) \prod_{i=2}^d P_{i|\pa(i)}(x_i | x_{\pa(i)}) , \label{eq:TreeFactorizationP}
\end{equation} 
where $\pa(i)$ (with $|\pa(i)|=1$) is the parent node of node $i$.

We   denote the KL-divergence  between distributions $Q$ and $P$  as $D(Q \| P) = \sum_{\brx} Q(\brx) \log \frac{Q(\brx)}{P(\brx)}$.  The set of distributions supported on $\cX$ is  denoted as $\cP(\cX)$.  Proofs of the theorems and the propositions are provided in the appendices, as part of the supplementary material.  For the sake of brevity, the presentation in the main body of the paper is focused on Ising tree models; the extension to the Gaussian case with experiments is deferred to App.~\ref{app:gauss}.

\subsection{System Model} \label{sec:SystemModel}
We consider binary random variables $X_i$ with alphabet $\cX=\{+1,-1\}$, where $1 \le i \le d$ and the joint distribution of $(X_1, \ldots, X_d)$ is represented by a tree-structured Ising model~\cite{BreslerKarzand18}. The observation for the $i$th node is represented by random variable $Y_i = X_i N_i$, where $N_i$ is   multiplicative binary noise with $\Pr(N_i = -1) = q_i$ and $\Pr(N_i = +1) = 1-q_i$. Thus, the observations are corrupted by independent but non-identical noise; that is, $q_i$ may differ for different values of $i$. We assume that our model satisfies the following properties:

\begin{enumerate} 
	\vspace{-0.06in}
	\item[P1:] (\emph{Zero external field}): The marginals for the hidden variables are uniform, i.e., $\Pr(X_i = 1)  = \Pr(X_i = -1) = 0.5$, for $1 \le i \le d$.
	
	\item[P2:] (\emph{Bounded Correlation}): If $\cE$ denotes the edge set of tree $\T$, and $\{i,j\} \in \cE$, then the correlations $\rho_{i,j} \triangleq \bbE[X_i X_j]$ are uniformly bounded as follows: $\rho_{\min} \le |\rho_{i,j}| \le \rho_{\max}$, where $0 < \rho_{\min} \le \rho_{\max} < 1$.
	
	\item[P3:] (\emph{Bounded Noise}): The noise crossover probability $q_i$, for $1 \le i \le d$, satisfies $0 \le q_i \le q_{\max} < 0.5$.
	\vspace{-0.06in}
\end{enumerate}
Properties P1 and P2 are common assumptions in the literature on learning Ising models~\cite{Shanmugam14NeurIPS, Scarlett16, Nikolakakis19AISTATS, BreslerKarzand18}, while P3 ensures that no node is independent of any other node due to noise~\cite{Katiyar20arxiv}.  For an Ising model with zero external field, the joint distribution given by~\eqref{eq:TreeFactorizationP} can be   expressed as~\cite{BreslerKarzand18}
\begin{equation}
	P(\brx) = \frac{1}{Z} \exp\bigg( \sum_{\{i,j\} \in \cE} \theta_{i,j} x_i x_j\bigg) , \label{eq:TreeFactorizationP_v2}
\end{equation}
where $\cE$ is the edge set of the   graph $\T$, and $Z$ is the normalization factor. For an Ising tree, if $\{i,j\} \in \cE$ then the interaction (exponential) parameter $\theta_{i,j}$ is related to the correlation $\rho_{i,j}=\mathbb{E}[X_i X_j]$ as~\citep[Lem.~A.3]{Nikolakakis19Predictive}
\begin{equation}
	\theta_{i,j} = \atanh(\rho_{i,j}) .\label{eq:ThetaAndRho}
\end{equation}
The bounded correlation property P2 then implies that $\atanh(\rho_{\min}) \le |\theta_{i,j}| \le \atanh(\rho_{\max})$.
Note that although property P1 is a common assumption that simplifies the presentation~\cite{Shanmugam14NeurIPS,Scarlett16, Nikolakakis19AISTATS, BreslerKarzand18}, the extension of our results  to the case where the marginals are not necessarily uniform can be readily obtained by following the approach outlined by~\citet{Katiyar20arxiv}.

\subsection{Definition of Equivalent Tree Structures} \label{sec:EquivClass}
We now define an equivalence relation on $d$-node trees. Let $\cT_d$ denote the set of all distinct trees on $d$ nodes. For a given tree graph, a \emph{leaf} is a node whose degree is one, i.e. a leaf node has only one neighbor. For $\T \in \cT_d$, let $\cL_{\T}$ denote set of leaf nodes in $\T$, $
	\cL_{\T} \triangleq \{ X_i : i ~\mbox{is a leaf node in }\T \}$.   Let $\sS_{\T}$ be the set of subsets of $\cL_{\T}$ such that no two nodes in a subset share a common neighbor,  i.e., 
{\small 
\begin{equation*}
	\sS_{\T} \triangleq \{ \cS \subseteq \cL_{\T}: \mbox{no two nodes in $\cS$ have the same neighbor} \}.
\end{equation*}
}
For a given $\cS \in \sS_{\T}$, let $\T_\cS$ denote the tree obtained from $\T$ by interchanging each node in $\cS$ with its corresponding neighbor in $\T$.\footnote{Note that when $\cS$ is the empty set, then $\T_\cS = \T$.} Define $[\T]$ to be the set of trees 
\begin{equation}
	[\T] \triangleq \{ T_\cS \,:\, \cS \in \sS_{\T}\}. \label{eq:Def_EquivalenceClassOfT}
\end{equation}
For $\hat{\T} \in \cT_d$, we say $\hat{\T} \sim \T$ if $\hat{\T} \in [\T]$.  The relation $\sim$ on $\cT_d$ is \emph{reflexive}, \emph{symmetric}, and \emph{transitive}, and is hence an \emph{equivalence relation}~\cite{HersteinBook75}. Therefore, with respect to relation $\sim$, the set $\cT_d$ is partitioned into disjoint equivalence classes, where $[\T]$ is the equivalence class of~$\T$.

\subsection{Problem Statement}
Let $\bry_1^n = \{\bry_1, \ldots , \bry_n\}$ denote $n$ independently sampled noisy observations, where the $i$th noisy sample is a $d$-dimensional column vector given by $\bry_i = (y_{i,1}, \ldots , y_{i,d})^T$ with $y_{i,j}$ denoting the $i$th observation corresponding to the $j$th node. As discussed in Sec.~\ref{sec:SystemModel}, we assume a binary multiplicative non-identical noise model at each node, with $y_{i,j} \in \cY \triangleq \{+1, -1\}$ and $\bry_i \in \cY^d$. Given $\bry_1^n$, a learning algorithm (or estimator) $\Psi : \cY^{d \times n} \to \cT_d$ provides an estimate of the underlying tree structure $\T$. We are interested in partial tree recovery (up to equivalence class $[\T]$), and an error is declared if the event $\big\{\!\Psi(\brY_1^n) \notin [\T]\big\}$ occurs.\footnote{Exact tree-structure identification cannot be guaranteed, in general, when the noise distribution across nodes is unknown and non-identical across nodes~\citep[Thm.~2]{Katiyar20arxiv}.} For this setup, the following questions are of interest:

(a) Quantify the number of samples {\em necessary} to achieve a target error probability (with any possible estimator).

(b) Quantify the number of samples {\em sufficient} to achieve a given error probability (using a particular estimator).
	
We answer (a)  in Section \ref{sec:Converse}  and (b) in Sections \ref{sec:Achievability_KA} and \ref{sec:SGA}.

\section{Impossibility Result} \label{sec:Converse}
For a given tree $\T = (\cV, \cE)$, and edge correlation bounds $0 < \rho_{\min} \le \rho_{\max} < 1$, let $\cP_{\T}(\rho_{\min}, \rho_{\max})$ denote the set of all tree-structured Ising models satisfying properties P1 and P2. Let the noise crossover probability at each node satisfy property P3, and hence be upper bounded by $q_{\max}$. Now, given $n$ independent noisy samples, the \emph{minimax error  probability} for partial tree structure recovery up to equivalence class $[\T]$, denoted $\cM_n(q_{\max}, \rho_{\min}, \rho_{\max})$, is 
\begin{equation}
\inf_{\Psi: \cY^{d \times n} \to \cT_d} \sup_{\substack{\T \in \cT_d, P \in \cP_{\T}(\rho_{\min}, \rho_{\max}),\\ 0\, \le \,q_i\, \le q_{\max} <\, 0.5}}\! \bbP_P\big(\Psi(\brY_1^n) \notin [\T] \big), \!\!\label{eq:Def_MinimaxError} 
\end{equation}
where $\bbP_P(\cdot)$ denotes the probability measure of the samples when the underlying tree distribution is $P$.  

\begin{theorem}[Necessary Samples for Partial Tree Recovery] \label{thm:NecessarySamples}
	Let $\rho_q \triangleq (1-2q_{\max}) \rho_{\min}$. If $d > 32$, and 
	 $n$ satisfies
	\begin{equation}
		n \,<\, \frac{\log d}{4\left(1-\rho_{\max}\right) \rho_q\, \atanh(\rho_q)} \, , \label{eq:NecessarySamples}
	\end{equation}
	then the minimax error    $\cM_n(q_{\max}, \rho_{\min}, \rho_{\max}) \ge 1/2$. In other words,    the optimal sample complexity satisfies $\Omega \left( (\log d )/ [(1- \rho_{\max})(1- 2q_{\max})^2\rho_{\min}^2  ]\right)$.
\end{theorem}
Theorem~\ref{thm:NecessarySamples} is proved by combining two key ingredients. The first is the choice of a sufficiently large number of tree structures that are relatively close each other, and satisfy the property that their respective equivalence classes are disjoint; see Fig.~\ref{fig:converse}. The second key ingredient is the choice of noise parameters for different nodes that have a relatively high impact on the error probability, while ensuring that the corresponding KL-divergence  is approximated by a closed-form expression; see Fig.~\ref{fig:noisy}. We note that   Theorem~\ref{thm:NecessarySamples} also provides an impossibility result for \emph{exact} tree structure recovery under non-identically distributed noise.

In a related work,  a bound on the necessary number of samples required for \emph{exact} tree structure recovery for the \emph{noiseless} setting was presented in~\citet[Thm.~3.1]{BreslerKarzand18}. Our result in Theorem~\ref{thm:NecessarySamples}, specialized to the case where  $q_{\max} = 0$, gives the same bound on the number of necessary samples as~\citet[Thm.~3.1]{BreslerKarzand18}. Therefore, in the noiseless setting, the partial tree structure recovery (up to equivalence class $[\T]$) is not  easier than {\em exact} structure recovery, in the minimax sense.\footnote{A related observation was made by~\citet{Scarlett16} in terms of comparison of the number of necessary samples for {\em partial} recovery (to within a given \emph{edit distance}), with the number of necessary samples for {\em exact} recovery, in the minimax sense.}

In another related work, \citet{Nikolakakis19AISTATS} provided a bound on the number of samples necessary for learning the \emph{exact} tree structure, under the assumption that the noise distribution for all the nodes is \emph{identical}. In particular,  for a tree-structured Ising model, it was shown~\citep[Thm.~2]{Nikolakakis19AISTATS} that the impact of noise gets manifested as a multiplicative factor $[1 - (4q(1-q))^d]^{-1}$, where $q$ denotes the noise crossover probability. This result implies that for any $q \in (0, 0.5)$, the impact of noise becomes negligible, i.e. the multiplicative factor tends to $1$, \emph{as the number of nodes $d$ tends to infinity}. In contrast, our result in Theorem~\ref{thm:NecessarySamples} for \emph{non-identical} noise, with $0 \le q_i \le q_{\max} < 0.5$   shows that the necessary number of samples for $q_{\max} > 0$ is strictly and uniformly greater than the number of samples for the noiseless setting by a multiplicative factor of \emph{at least $(1 - 2 q_{\max})^{-2}$  irrespective of how large the value of $d$ is}.

For a given $P \in \cP_{\T}(\rho_{\min}, \rho_{\max})$, the error probability $\bbP_P\big(\Psi(\brY_1^n) \notin [\T] \big)$ will depend on the specific tree structure~\cite{Tan11IT,TandonTZ20_JSAIT}, and the size of the equivalence class $[\T]$. The minimum and maximum possible size of the equivalence class, for a given number of nodes $d$, is quantified in App.~\ref{app:EqClassSize}.

\section{Algorithm by~\citet{Katiyar20arxiv} and our Improved Sufficiency Result} \label{sec:Achievability_KA}
Let $\T = (\cV, \cE)$ be a given tree, and let the underlying sample distribution and noise satisfy the properties in Sec.~\ref{sec:SystemModel}. Then, even if the noise is potentially large,\footnote{Note that $q_{\max}$ can be arbitrarily close to $0.5$.} and the noise statistics are completely unknown, the tree structure can be partially recovered (up to equivalence class $[\T]$) using the algorithm by~\citet{Katiyar20arxiv}. This algorithm for partial structure recovery of Ising tree models extends a previous method for partial recovery of Gaussian tree models using noisy samples~\citep{Katiyar19ICML}.
The cornerstone of the partial tree structure learning algorithm for Ising tree models~\citep{Katiyar20arxiv} is the classification of any set of $4$ distinct nodes in $\cV$ as \emph{non-star} or \emph{star}.\footnote{An overview of the algorithm classifying a set of $4$ nodes as non-star or star is presented in App.~\ref{app:KatiyarAlgoOverview}.}

\begin{definition}[Non-star and star~\citep{Katiyar20arxiv}]
	Any set of $4$ distinct nodes in $\cV$ forms a \emph{non-star} if there exists at least one edge in $\cE$ which, when removed, splits the tree into two sub-trees such that exactly $2$ of the $4$ nodes lie in one sub-tree and the other $2$ nodes lie in the other sub-tree. The nodes in the same sub-tree form a \emph{pair}. If the set is not a \emph{non-star}, it is categorized as a \emph{star}.
\end{definition}

A salient feature of the partial tree structure learning algorithm described by~\citet{Katiyar20arxiv} is that of ``proximal sets''. If we let  $t_1 \triangleq (1-2q_{\max})^2 \rho_{\min}^4$, and $t_2 \triangleq \min\big\{t_1, \frac{t_1 (1-2q_{\max})}{\rho_{\max}}\big\}$, then the proximal set of node $i$ is defined as the set of all nodes $j$ that satisfy $|\widehat{\rho}_{i,j}| \ge 0.5 t_2$, where $\widehat{\rho}_{i,j}$ denotes the empirical correlation between nodes $i$ and $j$. For making a star/non-star categorization, the algorithm by~\citet{Katiyar20arxiv} only considers nodes that lie within each others' proximal sets, and they use this property to prove the following theorem on the \emph{sufficient} number of noisy samples required to achieve a given error probability. The result is stated in terms $t_2$ and $\alpha \triangleq ( 1+\rho_{\max}^2)/{2}$.

\begin{theorem}[Sufficient Sample Complexity Bound~\cite{Katiyar20arxiv}] \label{thm:Katiyar_AchievabilityResult}
	The equivalence class $[\T]$ can be correctly recovered with probability at least $1 - \tau$ using the algorithm by~\citet{Katiyar20arxiv} if the number of noisy samples $n$ satisfy $n \ge \frac{128}{\delta^2} \log \big(\frac{6 d^2}{\tau} \big)$, where $\delta \triangleq \frac{t_2^3 (1-\alpha)}{128}$. 
\end{theorem}
From the statement of the above theorem, we observe that $\delta = \frac{t_2^3 (1-\alpha)}{128} \propto \rho_{\min}^{12}$ because $t_2$ scales linearly with $t_1$ while $t_1 \propto \rho_{\min}^{4}$, thereby implying that $n$ scales as $\rho_{\min}^{-24}$. In the following theorem, we show that the above result can be significantly improved, via a more refined analysis, by proving that $n$ only needs to scale as $\rho_{\min}^{-8}$.


\begin{theorem}[Improved   Sample Complexity Bound] \label{thm:ImprovedAchievabilityBound}
	The equivalence class $[\T]$ can be correctly recovered with probability at least $1 - \tau$ using the algorithm by~\citet{Katiyar20arxiv} if the number of samples satisfies	
\begin{equation}
	n \ge \frac{2}{\tilde{\delta}^2} \log \left(\frac{d^2}{\tau} \right)\quad \mbox{where} \quad \tilde{\delta} \triangleq \frac{t_2 (1-\alpha)}{20}. \label{eq:SGA_Achievability}
\end{equation}
\end{theorem}
Compared to Theorem~\ref{thm:Katiyar_AchievabilityResult}, this significantly improved result (because the right-hand-side is $O(1/\tilde{\delta}^2)$ instead of $O(1/\delta^2)$ and  $\delta=\tilde{\delta}^3$) is obtained by refining the  probability bounds for events such as 
$\frac{\widehat{\rho}_{1,3}\,\widehat{\rho}_{2,4}}{\widehat{\rho}_{1,2}\,\widehat{\rho}_{3,4}} < \alpha$  and $\frac{\widehat{\rho}_{1,3}\,\widehat{\rho}_{2,4}}{\widehat{\rho}_{1,4}\,\widehat{\rho}_{2,3}} > \alpha$.
	
\section{Symmetrized  Geometric Averaging (SGA) } \label{sec:SGA}
We now present a modified procedure for declaring a $4$-node sub-tree as a star or non-star. This algorithm is denoted SGA\_IS\_NON\_STAR (see Algorithm~\ref{alg:Mod_StarNonStar}), and is a symmetrized variant of the corresponding algorithm by~\citet{Katiyar20arxiv} with additional geometric averaging. The motivation behind SGA can be seen by considering an example where   $\{X_1,X_2,X_3,X_4\}$ forms a  non-star with pair $\{X_1,X_2\}$. If the noisy correlations are denoted $\tilde{\rho}_{i,j} \triangleq \bbE[Y_i Y_j]=(1-2q_i)(1-2q_j) \rho_{i,j}$, we have $\frac{\tilde{\rho}_{1,3}\,\tilde{\rho}_{2,4}}{\tilde{\rho}_{1,2}\,\tilde{\rho}_{3,4}} \le \rho_{\max}^2$ and $\frac{\tilde{\rho}_{1,4}\,\tilde{\rho}_{2,3}}{\tilde{\rho}_{1,2}\,\tilde{\rho}_{3,4}} \le \rho_{\max}^2$. Hence, we would expect the following metrics, based on empirical correlations, to satisfy
\begin{equation} \label{eq:SGA_metrics}
	\mathrm{(i)} \ \, \frac{\widehat{\rho}_{1,3}\,\widehat{\rho}_{2,4}}{\widehat{\rho}_{1,2}\,\widehat{\rho}_{3,4}} < \alpha \quad\mbox{and}\quad \mathrm{(ii)} \ \, \frac{\widehat{\rho}_{1,4}\,\widehat{\rho}_{2,3}}{\widehat{\rho}_{1,2}\,\widehat{\rho}_{3,4}} < \alpha .
\end{equation}
In contrast~\citet{Katiyar20arxiv} who checks  condition $\mathrm{(i)}$ in~\eqref{eq:SGA_metrics} but ignores $\mathrm{(ii)}$, the SGA variant compares the geometric average of the metrics in $\mathrm{(i)}$ and $\mathrm{(ii)}$ against the threshold $\alpha$ for checking if nodes $\{X_1,X_2\}$ form a pair.

\begin{algorithm}[tb]
	\caption{SGA\_IS\_NON\_STAR}
	\label{alg:Mod_StarNonStar}
	\begin{algorithmic}
		\STATE Let the set of $4$ nodes be $\{X_1,X_2,X_3,X_4\}$
		
		\STATE {\bfseries Input:} Empirical correlations $\widehat{\rho}_{i,j}, \ 1 \le i < j \le 4$, \ Threshold $\alpha = (1 + \rho_{\max}^2)/2 $.
		
		\STATE Let $v_2 = \frac{\sqrt{|\widehat{\rho}_{1,3}\,\widehat{\rho}_{2,4}\, \widehat{\rho}_{1,4}\,\widehat{\rho}_{2,3}|}}{|\widehat{\rho}_{1,2}\,\widehat{\rho}_{3,4}|}$, \ $v_3 = \frac{\sqrt{|\widehat{\rho}_{1,2}\,\widehat{\rho}_{3,4}\, \widehat{\rho}_{1,4}\,\widehat{\rho}_{2,3}|}}{|\widehat{\rho}_{1,3}\,\widehat{\rho}_{2,4}|}$, \ $v_4 = \frac{\sqrt{|\widehat{\rho}_{1,2}\,\widehat{\rho}_{3,4}\, \widehat{\rho}_{1,3}\,\widehat{\rho}_{2,4}|}}{|\widehat{\rho}_{1,4}\,\widehat{\rho}_{2,3}|}$
		
		\STATE Let $ v = \min_{2 \le i \le 4} v_i$ and $ i^* = \argmin_{2 \le i \le 4} v_i$

		\IF{$v < \alpha$}
		
		\STATE Declare Non-star where $\{X_1,X_{i^*}\}$ forms a pair
		
		\ELSE
		
		\STATE Declare Star
		
		\ENDIF
		
	\end{algorithmic}
\end{algorithm}

\begin{proposition}[Sufficient Sample Complexity Bound] \label{prop:ImprovedAchievabilityBound_v2}
	The equivalence class $[\T]$ can be correctly recovered with probability at least $1 - \tau$ using the SGA\_IS\_NON\_STAR procedure in Algorithm~\ref{alg:Mod_StarNonStar} if the number of samples satisfies~\eqref{eq:SGA_Achievability}.
\end{proposition}
Intuitively, we expect that taking the geometric average of the metrics $\mathrm{(i)}$ and $\mathrm{(ii)}$ in~\eqref{eq:SGA_metrics} reduces the effect of   noise, hence improving robustness. Although  we obtain the same   sample complexity bound in Prop.~\ref{prop:ImprovedAchievabilityBound_v2} as Theorem~\ref{thm:ImprovedAchievabilityBound}  because the Hoeffding's inequality used is rather loose (and not distribution dependent), the error exponent analysis in Sec.~\ref{sec:ErrExpAnalysis} highlights the   advantage of SGA. Furthermore, Monte Carlo simulations with a variety of tree structures and parameters demonstrate SGA's superior robustness in Sec.~\ref{sec:NumericalResults}.


 SGA and the algorithm by \citet{Katiyar20arxiv} are applicable to wider classes of models such as Gaussian graphical models in which node observations are corrupted by independent but non-identically distributed Gaussian noise \citep{Katiyar19ICML}. That is $\mathbf{X}=(X_1,\ldots, X_d)$ follows a zero-mean Gaussian with covariance matrix $\bm{\Sigma}^*$ and $(\bm{\Sigma}^*)^{-1}$ has sparsity pattern that corresponds to a tree $\T$.  However, we observe $\mathbf{Y}=(Y_1,\ldots, Y_d)$ with covariance matrix $\bm{\Sigma}^*+ \mathbf{D}^*$ where $\mathbf{D}^*$ is an unknown non-negative diagonal  matrix. If $\mathbf{D}^*$ is non-zero, the structure of $\T$ cannot be  identified in general. However, by computing the empirical correlations in an analogous fashion, we show in App.~\ref{app:gauss} that SGA is similarly robust vis-\`a-vis algorithms proposed in  \citet{Katiyar19ICML} and \citet{Katiyar20arxiv}.

\section{Error Exponent Analyses} \label{sec:ErrExpAnalysis}
The \emph{error exponent}, also called the \emph{inaccuracy rate}~\cite{Kester86}, captures the exponential decay of the error probability with $n$ as a function of the distribution. For a given tree with $d$ nodes  $\T \in \cT_d$ and graphical model $P \in \cP_{\T}(\rho_{\min}, \rho_{\max})$, let $\tilde{P}$ denote the joint distribution of the noisy variables where the noise crossover probability at the $i$th node satisfies $0 \le q_i \le q_{\max} < 0.5$. Then,  the error exponent of a given algorithm $\Psi$ is\footnote{For the estimators considered in this paper, it can be shown that the limit of the expression on the right side of \eqref{eq:Def_ErrorExp} exists.} 
\begin{align}
	E(\Psi,\tilde{P})  &=	E(\Psi,\tilde{P},q_{\max},\rho_{\min},\rho_{\max})   \\
	&\triangleq  \liminf_{n \to \infty} -\frac{1}{n} \log \bbP_{\tilde{P}}\big(\Psi(\brY_1^n) \notin [\T] \big).\label{eq:Def_ErrorExp} 
\end{align}
We   label the estimator by~\citet{Katiyar20arxiv} as $\Psi_{\KA}$ which uses Algorithm~\ref{alg:Kat_StarNonStar} in App.~\ref{app:KatiyarAlgoOverview}  for declaring a set of~$4$ nodes as star or non-star.  We also label the estimator employing the SGA algorithm described in Algorithm~\ref{alg:Mod_StarNonStar} as $\Psi_{\SGA}$. In the following, we use Sanov's theorem~\cite{sanov57} to quantify the error exponents of $\Psi_{\KA}$ and $\Psi_{\SGA}$, and demonstrate that, in general, $\Psi_{\SGA}$ provides a better (i.e., higher) error exponent compared to $\Psi_{\KA}$.

\subsection{Error exponent using $\Psi_{\KA}$}
The performance of $\Psi_{\KA}$ depends on its ability to correctly declare a set of $4$ nodes as star or non-star (with the appropriate pairing of nodes). The following proposition characterizes the error exponent for a $4$-node tree.

\begin{proposition} \label{prop:KA_ErrExp}
	Let $P \in \cP_{\T}(\rho_{\min}, \rho_{\max})$ be a tree for $4$ nodes $\{X_1,X_2,X_3,X_4\}$, and let $\tilde{P}$ denote the joint distribution of the noisy variables. Let $\rho_{j,k}^{(Q)} \triangleq \mathbb{E}_Q[X_j X_k]$.
	
	(a) If the tree distribution $P$ corresponds to the Markov chain $X_1 \myrule X_2 \myrule X_3 \myrule X_4$, and we define
		\begin{align}
		e_1 &\triangleq \min_{Q \in \cP(\cX^4)} \Big\{D(Q \| \tilde{P}) \, : \, \frac{\rho_{1,3}^{(Q)}\,\rho_{2,4}^{(Q)}}{\rho_{1,2}^{(Q)}\,\rho_{3,4}^{(Q)}} \ge \alpha \Big\}, \label{eq:KA_e1} \\
		e_2 &\triangleq \min_{Q \in \cP(\cX^4)} \Big\{D(Q \| \tilde{P}) \, : \, \frac{\rho_{1,3}^{(Q)}\,\rho_{2,4}^{(Q)}}{\rho_{1,4}^{(Q)}\,\rho_{2,3}^{(Q)}} \le \alpha \Big\}, \label{eq:KA_e2} 
	\end{align}
\vspace{-1mm}
then we have $E(\Psi_{\KA},\tilde{P}) = \min\{e_1,e_2\}$.
	
\vspace{1mm}	
(b) If $P$ corresponds to a star tree structure, then $E(\Psi_{\KA},\tilde{P})$ can be expressed as
	{\small{
	\begin{equation}
		\min_{Q \in \cP(\cX^4)} \Big\{D(Q \| \tilde{P}) :  \frac{\rho_{1,3}^{(Q)}\,\rho_{2,4}^{(Q)}}{\rho_{1,2}^{(Q)}\,\rho_{3,4}^{(Q)}} \le \alpha, \, \frac{\rho_{1,3}^{(Q)}\,\rho_{2,4}^{(Q)}}{\rho_{1,4}^{(Q)}\,\rho_{2,3}^{(Q)}} \ge \alpha \Big\}. \label{eq:KA_ErrExpStar}
	\end{equation}
	}}
\end{proposition}
For a tree $\T$ with $d \ge 4$ nodes, if  $\{X_{i_j}\}_{j=1}^4$ are $4$ nodes in $\T$ that form a star structure (resp.\ non-star structure with pair $\{X_{i_1},X_{i_2}\}$), and $\tilde{P}$ denotes the distribution of the noisy variables $\{Y_{i_j}\}_{j=1}^4$, then the exponent corresponding to an incorrect decision on the structure of these nodes by the procedure in Algorithm~\ref{alg:Kat_StarNonStar} is equal to expression on the right side of \eqref{eq:KA_ErrExpStar} (resp.\  equal to $\min\{e_1,e_2\}$ with $e_1$, $e_2$ defined in \eqref{eq:KA_e1}, \eqref{eq:KA_e2}), where $\rho_{j,k}^{(Q)} = \mathbb{E}_Q[X_{i_j} X_{i_k}]$.

\subsection{Error exponent using $\Psi_{\SGA}$}
The following proposition characterizes the error exponent using $\Psi_{\SGA}$ for a $4$-node tree.

\begin{proposition} \label{prop:SGA_ErrExp}
		Let $P \in \cP_{\T}(\rho_{\min}, \rho_{\max})$ be a tree for $4$ nodes $\{X_1,X_2,X_3,X_4\}$, and let $\tilde{P}$ denote the   joint distribution of the noisy variables.
	
	(a) If the tree distribution $P$ corresponds to the Markov chain $X_1 \myrule X_2 \myrule X_3 \myrule X_4$, and we define
	{\small{
		\begin{align}
		e_3 &\triangleq \! \min_{Q \in \cP(\cX^4)} \Big\{D(Q \| \tilde{P})  : \! \frac{\sqrt{|\rho_{1,3}^{(Q)}\rho_{2,4}^{(Q)}\rho_{1,4}^{(Q)}\rho_{2,3}^{(Q)}|}}{|\rho_{1,2}^{(Q)}\rho_{3,4}^{(Q)}|} \!\ge\! \alpha \Big\}, \label{eq:SGA_e3}\\
		e_4 &\triangleq \min_{Q \in \cP(\cX^4)} \left\{D(Q \| \tilde{P})  :  |\rho_{1,3}^{(Q)}\rho_{2,4}^{(Q)}| \ge |\rho_{1,2}^{(Q)}\rho_{3,4}^{(Q)}| \right\}, \label{eq:SGA_e4} \\
		e_5 &\triangleq \min_{Q \in \cP(\cX^4)} \left\{D(Q \| \tilde{P})  :  |\rho_{1,4}^{(Q)}\rho_{2,3}^{(Q)}| \ge |\rho_{1,2}^{(Q)}\rho_{3,4}^{(Q)}| \right\}, \label{eq:SGA_e5}
	\end{align}
    }}
	then we have $E(\Psi_{\SGA},\tilde{P}) = \min\{e_3,e_4,e_5\}$.
	
	\vspace{1mm}	
	(b) If $P$ corresponds to a star tree structure, then $E(\Psi_{\SGA},\tilde{P})$ can be expressed as
{\small{
	\begin{equation}
	\min_{Q \in \cP(\cX^4)} \Big\{D(Q \| \tilde{P}) \, : \, \frac{\sqrt{|\rho_{1,3}^{(Q)}\,\rho_{2,4}^{(Q)}\,\rho_{1,4}^{(Q)}\,\rho_{2,3}^{(Q)}|}}{|\rho_{1,2}^{(Q)}\,\rho_{3,4}^{(Q)}|} \le \alpha\Big\}. \label{eq:SGA_ErrExpStar}
\end{equation}
}}	
\end{proposition}
For a   tree $\T$ with $d \ge 4$ nodes, if $\{X_{i_j}\}_{j=1}^4$ are $4$ nodes in $\T$ that form a star   (resp.\ non-star   with pair $\{X_{i_1},X_{i_2}\}$), and $\tilde{P}$ denotes the distribution of the noisy variables $\{Y_{i_j}\}_{j=1}^4$, then the exponent corresponding to an incorrect decision on the structure of these nodes by the procedure in Algorithm~\ref{alg:Mod_StarNonStar} is equal to expression on the right side of \eqref{eq:SGA_ErrExpStar} (resp.\   equal to $\min\{e_3,e_4,e_5\}$ with $e_3$, $e_4$, $e_5$ defined in \eqref{eq:SGA_e3}, \eqref{eq:SGA_e4},  and~\eqref{eq:SGA_e5}), where $\rho_{j,k}^{(Q)} = \mathbb{E}_Q[X_{i_j} X_{i_k}]$.

\vspace{-0.1in}
\subsection{Numerical comparison of the error exponents} \label{sec:ErrExpNumerical}
\vspace{-0.05in}

Because the expressions for the error exponents in Props.~\ref{prop:KA_ErrExp} and~\ref{prop:SGA_ErrExp} are not easily comparable (since $E(\Psi_{\KA},\tilde{P})$ and $E(\Psi_{\SGA},\tilde{P})$ are non-convex optimization problems), we present  numerical comparisons of $E(\Psi_{\KA},\tilde{P})$ and $E(\Psi_{\SGA},\tilde{P})$  for $4$-node homogeneous trees. Fig.~\ref{Fig:ErrExp_Chain} compares the error exponents  for a $4$-node Markov chain $X_1 \myrule X_2 \myrule X_3 \myrule X_4$ where all the edge correlations are same, and are denoted as $\rho$. Fig.~\ref{Fig:ErrExp_Chain}(a) considers a noiseless scenario where $q_{\max}=0$, and it is observed that the error exponent for $\Psi_{\SGA}$ is \emph{significantly} higher (hence better) than that for $\Psi_{\KA}$ for small values of $\rho$; e.g., when $\rho < 0.6$. On the other hand, $\Psi_{\KA}$ has only \emph{marginally} higher exponent for higher values of $\rho$ (when $\rho > 0.71$). Fig.~\ref{Fig:ErrExp_Chain}(b) compares the error exponents for the scenario where $\rho = 0.74$ is fixed, and where the noise crossover probabilities for the nodes satisfy $q_1 = q_2 = q_3 = 0$ and $q_4 = q_{\max}$. It is seen that although $\Psi_{\KA}$ has better exponent than $\Psi_{\SGA}$ in the neighborhood of $q_{\max} = 0$, the performance of $\Psi_{\SGA}$ is slightly better for relatively higher values of $q_{\max}$.

\begin{figure}[t]
	\vskip 0.2in
	\begin{center}
		\centerline{\includegraphics[width=\columnwidth]{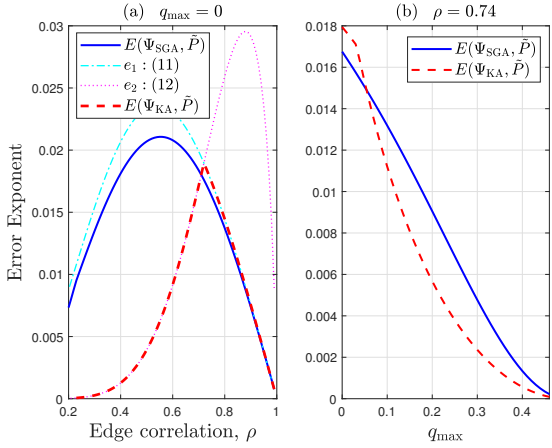}}
	\caption{Error exponents for a $4$-node homogeneous chain where all tree edges have correlation $\rho$. \ (a)~Noiseless setting, $q_{\max} = 0$, \ (b)~Error exponent versus $q_{\max}$ for fixed $\rho = 0.74$.}
	\label{Fig:ErrExp_Chain}
	\end{center}
	\vskip -0.2in
\end{figure}
\begin{figure}[t]
	\vskip 0.2in
	\begin{center}
		\centerline{\includegraphics[width=\columnwidth]{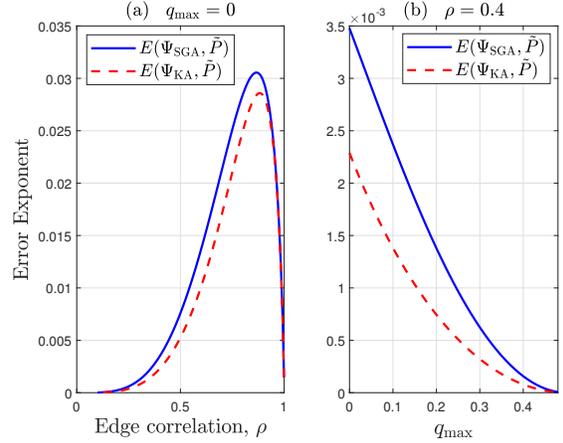}}
	\caption{Error exponents for a $4$-node star-structured tree where all edges have correlation $\rho$. \ (a)~Noiseless setting, $q_{\max} = 0$, \ (b)~Error exponent versus $q_{\max}$ for fixed $\rho = 0.4$.}
	\label{Fig:ErrExp_Star}
	\end{center}
	\vskip -0.2in
\end{figure}

Fig.~\ref{Fig:ErrExp_Star} compares the error exponents with $\Psi_{\KA}$ and $\Psi_{\SGA}$ for a $4$-node star-structured tree where all edge correlations are same and are denoted $\rho$. Fig.~\ref{Fig:ErrExp_Star}(a) considers a noiseless scenario where $q_{\max}=0$, and it is observed that the error exponent for $\Psi_{\SGA}$ is slightly higher than the error exponent for $\Psi_{\KA}$ for all values of $\rho$. Fig.~\ref{Fig:ErrExp_Star}(b) compares the error exponents for the scenario where $\rho = 0.4$ is fixed, and where the noise crossover probabilities for the nodes satisfy $q_1 = q_2 = q_3 = 0$ and $q_4 = q_{\max}$. Again, we see that the error exponent for $\Psi_{\SGA}$ is better than that for $\Psi_{\KA}$.

Monte Carlo simulations for $4$-node homogeneous trees (i.e., trees with equal correlations on the edges) that corroborate the theoretical results in this section, are presented in App.~\ref{app:4nodeHomogeneousTrees}. Even though Figs.~\ref{Fig:ErrExp_Chain} and \ref{Fig:ErrExp_Star} suggest that $\Psi_{\KA}$ sometimes outperforms $\Psi_{\SGA}$, we show  for larger trees that the performance of $\Psi_{\SGA}$ is almost always better than~$\Psi_{\KA}$. We explain why this is so in Sec.~\ref{sec:numerical_chain}. Since Sanov's theorem is also applicable to random variables with arbitrary alphabets~\citep[Ch.~3]{deuschel00}, we expect  similar error exponent performances for Gaussian models.

\vspace{-0.1in}
\section{Numerical Results} \label{sec:NumericalResults}
\vspace{-0.05in}

In this section, we present Monte Carlo simulation results for $12$-node trees with three different tree structures: (i) Chain, (ii) Hybrid, (iii) Star (see Fig.~\ref{Fig:3TreeStructures} in App.~\ref{app:3TreeStructures}). The chain and the star structures are known to be extremal in terms of the error probability~\cite{Tan10TSP,TandonTZ20_JSAIT}, while the hybrid tree structure is a combination of the chain and star structures. For a given tree structure $\T$, and $n$ noisy samples $\brY_1^n$, the error probability $\bbP\big(\Psi(\brY_1^n) \notin [\T] \big)$ for a given learning algorithm $\Psi$, is estimated using $10^5$ iterations (or runs) in the Monte Carlo simulation framework, where an error is declared if the estimated tree does not belong to the equivalence class $[\T]$. We obtain error probability results for three different   algorithms $\Psi_{\KA}$, $\Psi_{\SGA}$ and $\Psi_{\CL}$, the classical Chow-Liu  tree learning algorithm~\cite{ChowLiu68}.  
For   $\Psi_{\KA}$ and $\Psi_{\SGA}$,   knowledge of $\rho_{\min}$ and $\rho_{\max}$ is assumed. The source code to reproduce the experiments is included in our submission.
\vspace{-0.1in}
\subsection{Error Probabilities for $12$-node chain} \label{sec:numerical_chain}
\vspace{-0.05in}

Fig.~\ref{Fig:ErrProb_12Chain} compares the error probabilities for a $12$-node Markov chain, where all edge correlations are equal to $\rho$, using three learning algorithms: $\Psi_{\KA}$, $\Psi_{\SGA}$, and $\Psi_{\CL}$. Fig.~\ref{Fig:ErrProb_12Chain}(a) plots the results for the noiseless setting, $q_{\max} = 0$, with $\rho=0.8$. In this case, it is seen that the Chow-Liu algorithm $\Psi_{\CL}$ provides the minimum error probability as it the maximum-likelihood algorithm for the noiseless setting~\cite{ChowWagner73,Tan11IT}. The observation that $\Psi_{\SGA}$ has lower error probability than $\Psi_{\KA}$ for $\rho=0.8$, $q_{\max}=0$ can be intuitively explained as follows. An error event using $\Psi_{\KA}$ or $\Psi_{\SGA}$ occurs when a set of $4$ nodes in the tree is incorrectly declared as star/non-star (see Sec.~\ref{sec:Achievability_KA} and Sec.~\ref{sec:SGA}, respectively). Now, in the process of building a tree structure, these algorithms may pick a set of $4$ non-neighboring nodes (that belong to each others' proximal sets) to characterize them as star or non-star. For instance, consider the set of $4$ nodes $X_1, X_3, X_5, X_7$ that forms a sub-chain (of the $12$-node chain) where the effective edge correlation for the sub-chain is $0.8^2 = 0.64$. From Fig.~\ref{Fig:ErrExp_Chain}(a), we note that $\Psi_{\SGA}$ has a significantly higher exponent than $\Psi_{\KA}$ when the edge correlation is $0.64$, and therefore we would expect $\Psi_{\SGA}$ to have a lower error probability compared to $\Psi_{\KA}$ when characterizing these $4$ nodes as star or non-star.

\begin{figure}[t]
	\vskip 0.2in
	\begin{center}
		\centerline{\includegraphics[width=\columnwidth]{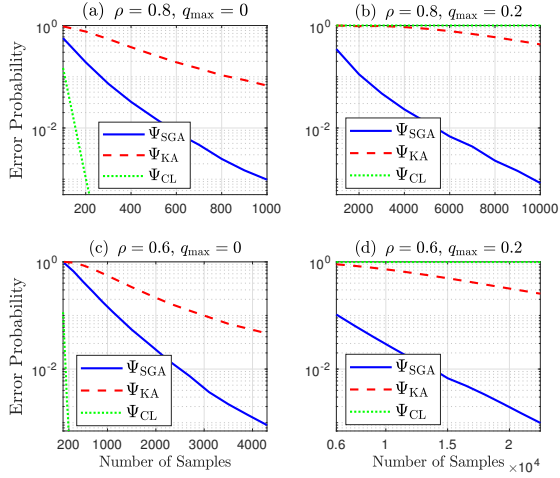}}
	\caption{Comparison of error probabilities for a $12$-node Markov chain where all edge correlations are equal to $\rho$.}
	\label{Fig:ErrProb_12Chain}
	\end{center}
	\vskip -0.2in
\end{figure}

Fig.~\ref{Fig:ErrProb_12Chain}(b) compares the error probabilities when $\rho=0.8$ for the noisy case, where noise is added to alternate nodes, i.e.,  $q_i = q_{\max}=0.2$ for $i \in  O_{12} \triangleq \{1,3,5,7,9,11\}$, while $q_j = 0$ for $j \in E_{12}\triangleq\{2,4,6,8,10,12\}$. In contrast to the noiseless case,  we observe from Fig.~\ref{Fig:ErrProb_12Chain}(b) that $\Psi_{\CL}$ performs extremely poorly with error probability roughly equal to~$1$. Such a poor performance is expected of $\Psi_{\CL}$ because $\bbE[Y_4 Y_6] = 0.64 > 0.48 = \bbE[Y_4 Y_5] = \bbE[Y_5 Y_6]$, and hence the tree estimated using $\Psi_{\CL}$ is more likely to pick the incorrect edge $\{X_4,X_6\}$ over the correct edges $\{X_4,X_5\}$ and $\{X_5,X_6\}$. Similar to Fig.~\ref{Fig:ErrProb_12Chain}(a), the plots in Fig.~\ref{Fig:ErrProb_12Chain}(b) highlight the clear superiority of $\Psi_{\SGA}$ over $\Psi_{\KA}$. A similar robust performance of $\Psi_{\SGA}$ is observed in the plots in Figs.~\ref{Fig:ErrProb_12Chain}(c)~and \ref{Fig:ErrProb_12Chain}(d) that compare the error probabilities for a $12$-node Markov chain where $\rho = 0.6$.

\vspace{-0.1in}
\subsection{Error Probabilities for $12$-node hybrid tree}
\vspace{-0.05in}

\begin{figure}[t]
	\vskip 0.2in
	\begin{center}
		\centerline{\includegraphics[width=\columnwidth]{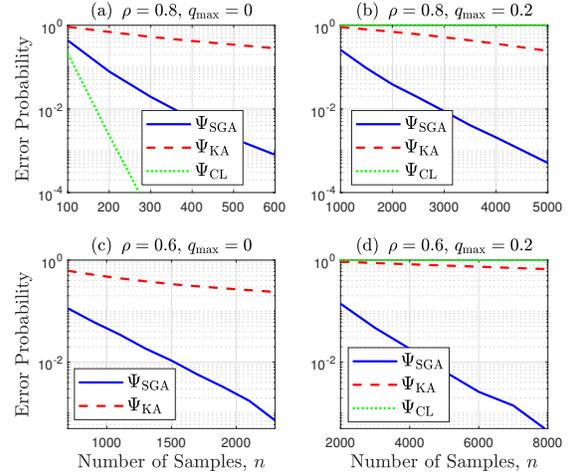}}
	\caption{Comparison of error probabilities for a $12$-node hybrid tree where all edge correlations are equal to $\rho$.}
	\label{Fig:ErrProb_12Hybrid}
	\end{center}
	\vskip -0.2in
\end{figure}

Fig.~\ref{Fig:ErrProb_12Hybrid} compares the error probabilities for a $12$-node hybrid tree where all   correlations are equal to $\rho$. Fig.~\ref{Fig:ErrProb_12Hybrid}(a) plots the results for the noiseless setting with $\rho =0.8$, while Fig.~\ref{Fig:ErrProb_12Hybrid}(b) considers the noisy case where noise is only added to even nodes, i.e. $q_i =0$ for $i \in O_{12}$, and $q_j =  q_{\max}=0.2$ for $j \in E_{12}$. For the noiseless case, as expected,   $\Psi_{\CL}$ provides the minimum error probability. However, for the noisy case, $\Psi_{\CL}$ performs poorly with error probability $\approx 1$. Again, this is expected because $\bbE[Y_3 Y_5] = 0.64 > 0.48 = \bbE[Y_3 Y_4] = \bbE[Y_4 Y_5]$, and hence the tree estimated using $\Psi_{\CL}$ is more likely to pick the incorrect edge $\{X_3,X_5\}$ over the correct edges $\{X_3,X_4\}$ and $\{X_4,X_5\}$ (see Fig.~\ref{Fig:3TreeStructures}(ii) in App.~\ref{app:3TreeStructures}). 

The plots in Fig.~\ref{Fig:ErrProb_12Hybrid}(c)~and (d)~compare the error probabilities for a hybrid tree when the edge correlation is $\rho = 0.6$. For the noiseless case in (c), the error probability using $\Psi_{\CL}$ is not plotted because it results in zero errors over $10^5$ Monte Carlo simulation runs for the given values of $n$ in Fig.~\ref{Fig:ErrProb_12Hybrid}(c). On the other hand, Fig.~\ref{Fig:ErrProb_12Hybrid}(d) considers the noisy case where noise is only added to even nodes, i.e. $q_i =0$ for $i \in O_{12}$, and $q_j =  q_{\max}=0.2$ for $j \in E_{12}$. The error probability using $\Psi_{\CL}$ is quite high for the noisy case (note, for instance, that $\bbE[Y_3 Y_5] = \bbE[Y_3 Y_4] = \bbE[Y_4 Y_5] = 0.36$). Fig.~\ref{Fig:ErrProb_12Hybrid} clearly highlights the robustness of $\Psi_{\SGA}$   over $\Psi_{\KA}$ and $\Psi_{\CL}$  when the underlying tree   is a hybrid of the chain and the star.

\vspace{-0.1in}
\subsection{Error Probabilities for $12$-node star tree}
\vspace{-0.05in}

\begin{figure}[t]
	\vskip 0.2in
	\begin{center}
		\centerline{\includegraphics[width=\columnwidth]{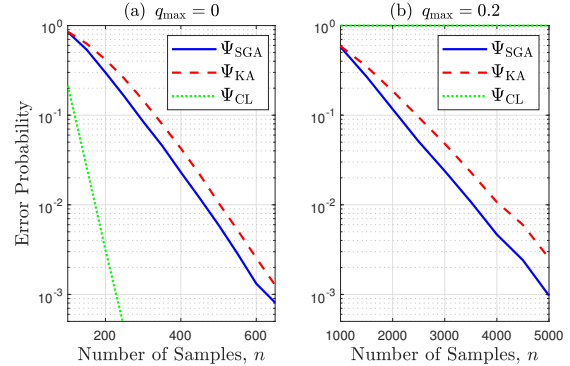}}
	\caption{Comparison of error probabilities for a $12$-node star, where all edge correlations are equal to $\rho = 0.6$.}
	\label{Fig:ErrProb_12Star_tPoint2}
	\end{center}
	\vskip -0.2in
\end{figure}

Fig.~\ref{Fig:ErrProb_12Star_tPoint2} compares the error probabilities using $\Psi_{\KA}$, $\Psi_{\SGA}$, and $\Psi_{\CL}$, for a $12$-node star where all correlations are equal to $\rho = 0.6$. Fig.~\ref{Fig:ErrProb_12Star_tPoint2}(a) plots the results for the noiseless setting ($q_{\max} = 0$), while Fig.~\ref{Fig:ErrProb_12Star_tPoint2}(b) considers the noisy case where noise is only added to odd nodes, i.e. $q_i = q_{\max} = 0.2$ for $i \in O_{12}$, while $q_j = 0$ for $j \in E_{12}$. The Chow-Liu algorithm $\Psi_{\CL}$ performs well in the noiseless setting, but fails miserably in the noisy setting. For both the noisy and noiseless settings, $\Psi_{\SGA}$ performs slightly better than   $\Psi_{\KA}$. This is justified using the error exponent result in Fig.~\ref{Fig:ErrExp_Star}(a) where it is observed that  $E(\Psi_{\SGA},\tilde{P})$ is slightly higher than that for $E(\Psi_{\KA},\tilde{P})$ when $\rho=0.6$.

The numerical results in this section demonstrate the robustness and superiority of   $\Psi_{\SGA}$ over   $\Psi_{\KA}$. 
Additional numerical results  for Gaussian trees are presented in App.~\ref{sec:gauss_num}.

\vspace{-0.1in}
\section{Discussion and Future Work}
\vspace{-0.05in}

There are several promising avenues for future research. First, SGA and the algorithm by \citet{Katiyar20arxiv} depend on the knowledge of  $\rho_{\max}$ through $\alpha$. Designing algorithms that do not depend on $\rho_{\max}$ would be of practical interest. Second, we can  tighten the sample complexity bounds so that the dependencies on the parameters $(\rho_{\min},\rho_{\max},q_{\max})$ are tightened. Finally, given noisy samples, we can endeavor to define equivalence classes (analogous to $[\T]$ here) and propose algorithms for the learning of various other graph structures such as random graphs \citep{Anandkumar12}, latent trees \citep{Choi11}, or forests   \citep{Tan11}.



%
%
%
%
\bibliography{abrv,conf_abrv,mybibfile}
\bibliographystyle{icml2020}

\appendix

\onecolumn

\section{Proof of Theorem~\ref{thm:NecessarySamples}} \label{app:NecessarySamples}
We will prove the result by judiciously choosing a sufficiently large subset of tree-structured graphs whose corresponding distributions are close enough (with respect to the KL-divergence ``metric''), and then applying   Fano's lemma~\cite{CoverBook06}. 

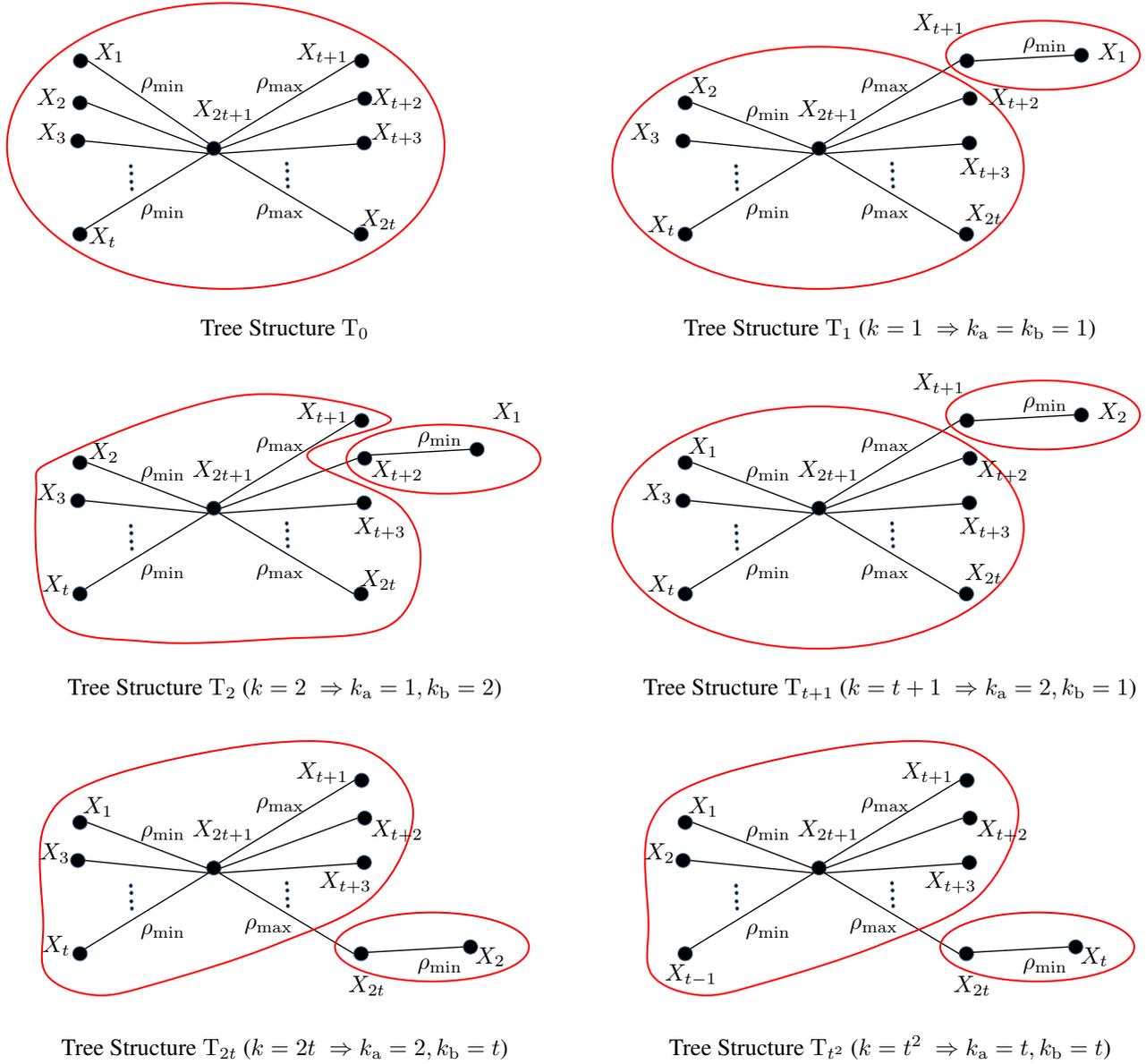
\begin{figure}
\centering
\begin{tabular}{cc}
\begin{overpic}[width=.5\textwidth]{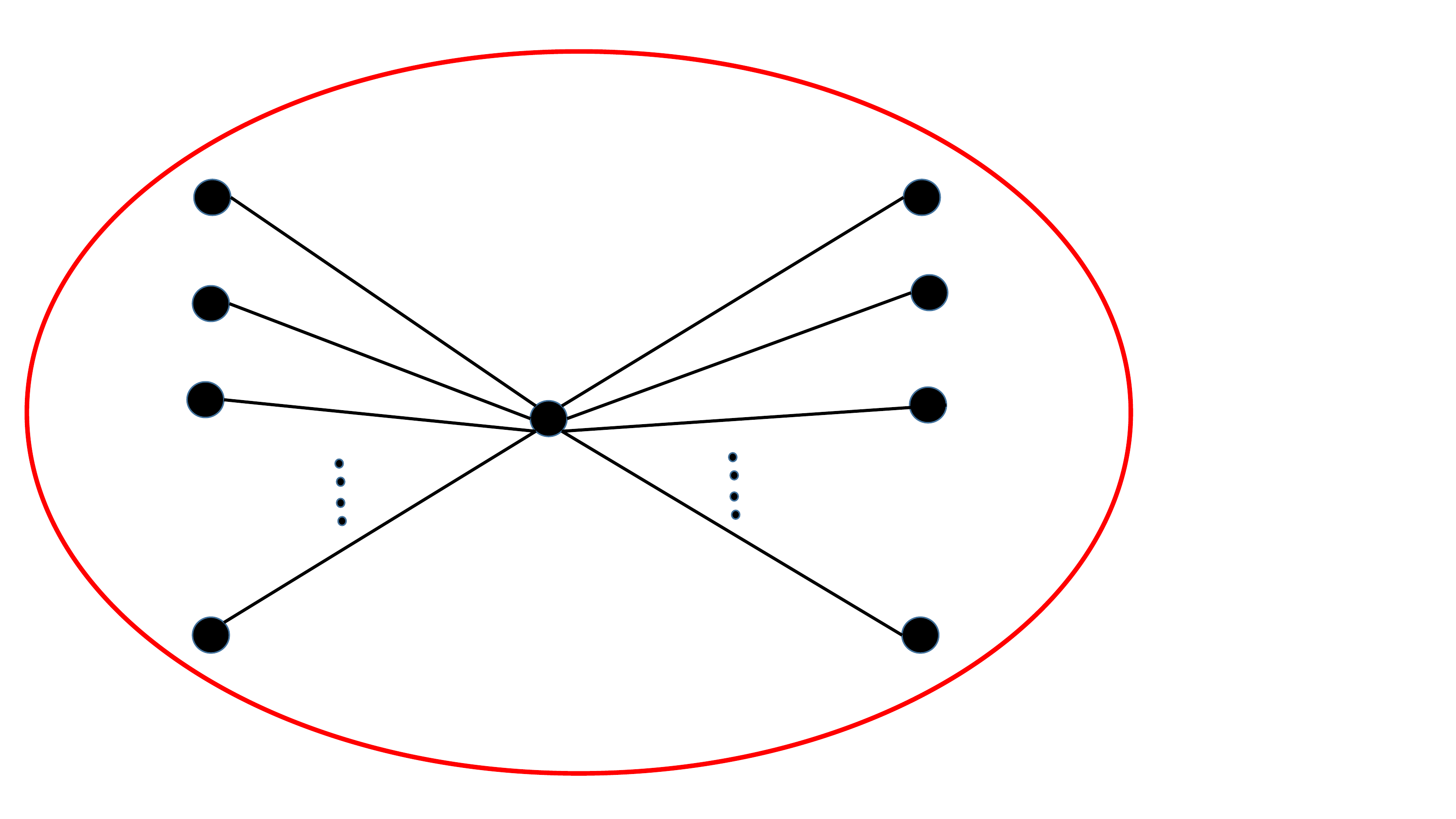}
\put(16,11){$X_t$}
\put(7,29){$X_3$}
\put(7,35){$X_2$}
\put(17,43){$X_1$}
\put(63,14){$X_{2t}$}
\put(65,29){$X_{t+3}$}
\put(65,35){$X_{t+2}$}
\put(52,43){$X_{t+1}$}
\put(34,33){$X_{2t+1}$}
\put(25,38){$\rho_{\min}$}
\put(25,16){$\rho_{\min}$}
\put(45,38){$\rho_{\max}$}
\put(45,16){$\rho_{\max}$}
\end{overpic} &
\begin{overpic}[width=.5\textwidth]{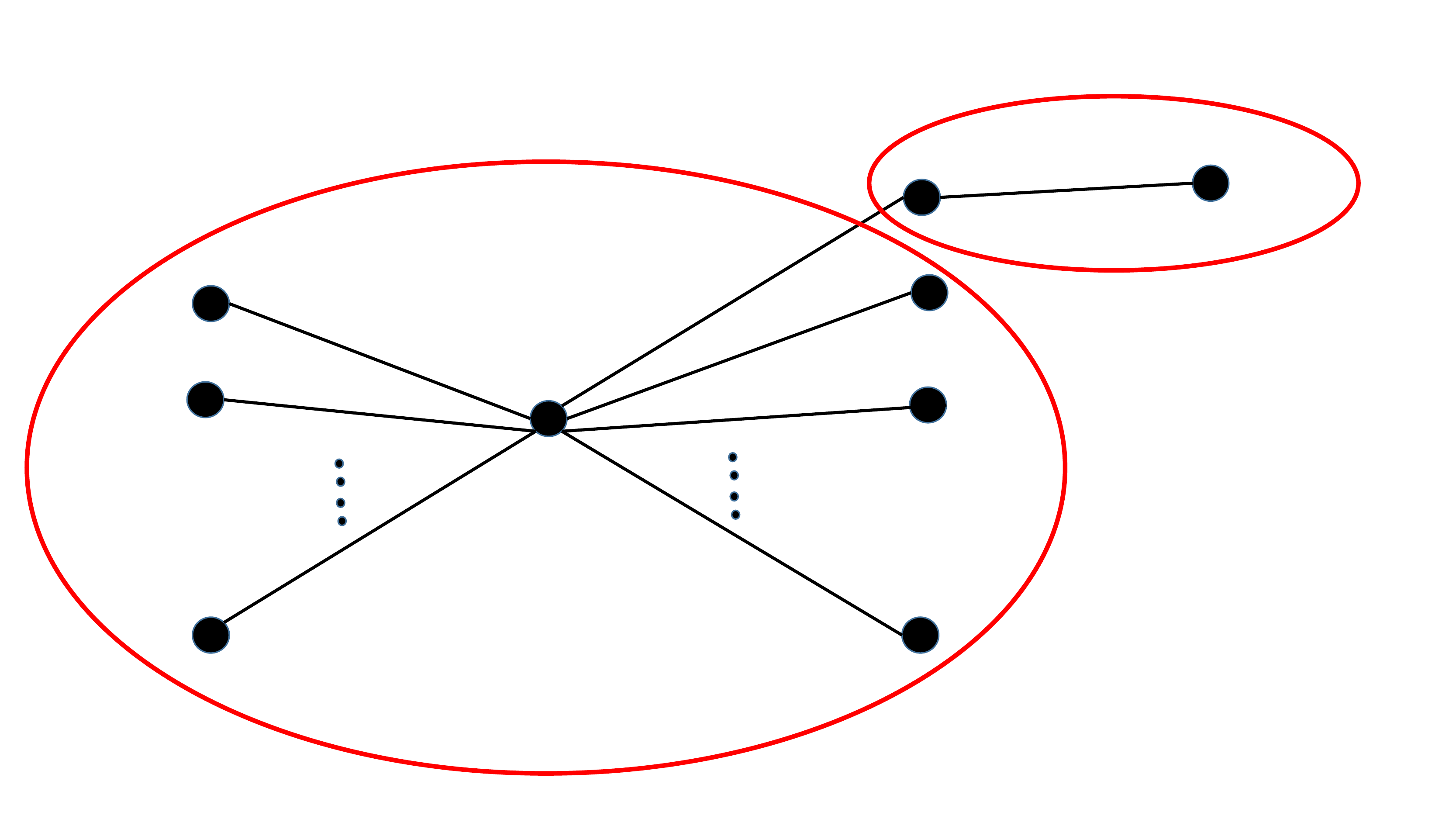}
\put(8,13){$X_t$}
\put(5,29){$X_3$}
\put(15,37){$X_2$}
\put(86,43){$X_1$}
\put(73,45){$\rho_{\min}$}
\put(63,15){$X_{2t}$}
\put(62,23){$X_{t+3}$}
\put(67,35){$X_{t+2}$}
\put(54,48){$X_{t+1}$}
\put(34,33){$X_{2t+1}$}
\put(25,33){$\rho_{\min}$}
\put(25,16){$\rho_{\min}$}
\put(45,38){$\rho_{\max}$}
\put(45,16){$\rho_{\max}$}
\end{overpic} 
\\
Tree Structure $\T_0$ & Tree Structure $\T_1$ ($k=1\;\Rightarrow k_{\mathrm{a}}=k_{\mathrm{b}}=1$) \\
\begin{overpic}[width=.5\textwidth]{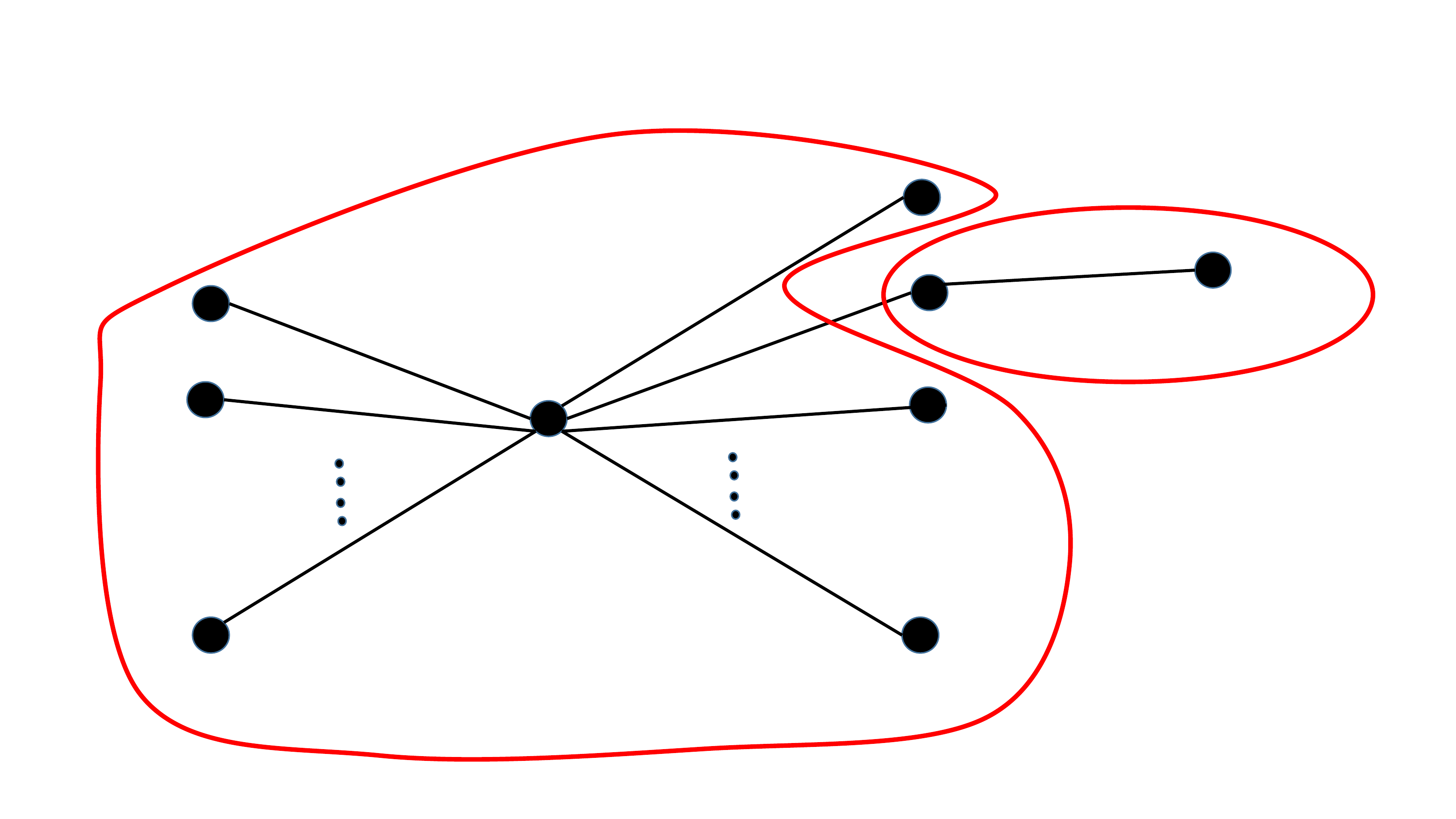}
\put(8,13){$X_t$}
\put(7,29){$X_3$}
\put(16,36){$X_2$}
\put(86,43){$X_1$}
\put(63,14){$X_{2t}$}
\put(62,23){$X_{t+3}$}
\put(65,33){$X_{t+2}$}
\put(52,43){$X_{t+1}$}
\put(34,33){$X_{2t+1}$}
\put(73,39){$\rho_{\min}$}
\put(25,33){$\rho_{\min}$}
\put(25,16){$\rho_{\min}$}
\put(45,38){$\rho_{\max}$}
\put(45,16){$\rho_{\max}$}
\end{overpic} 
 &
 \begin{overpic}[width=.5\textwidth]{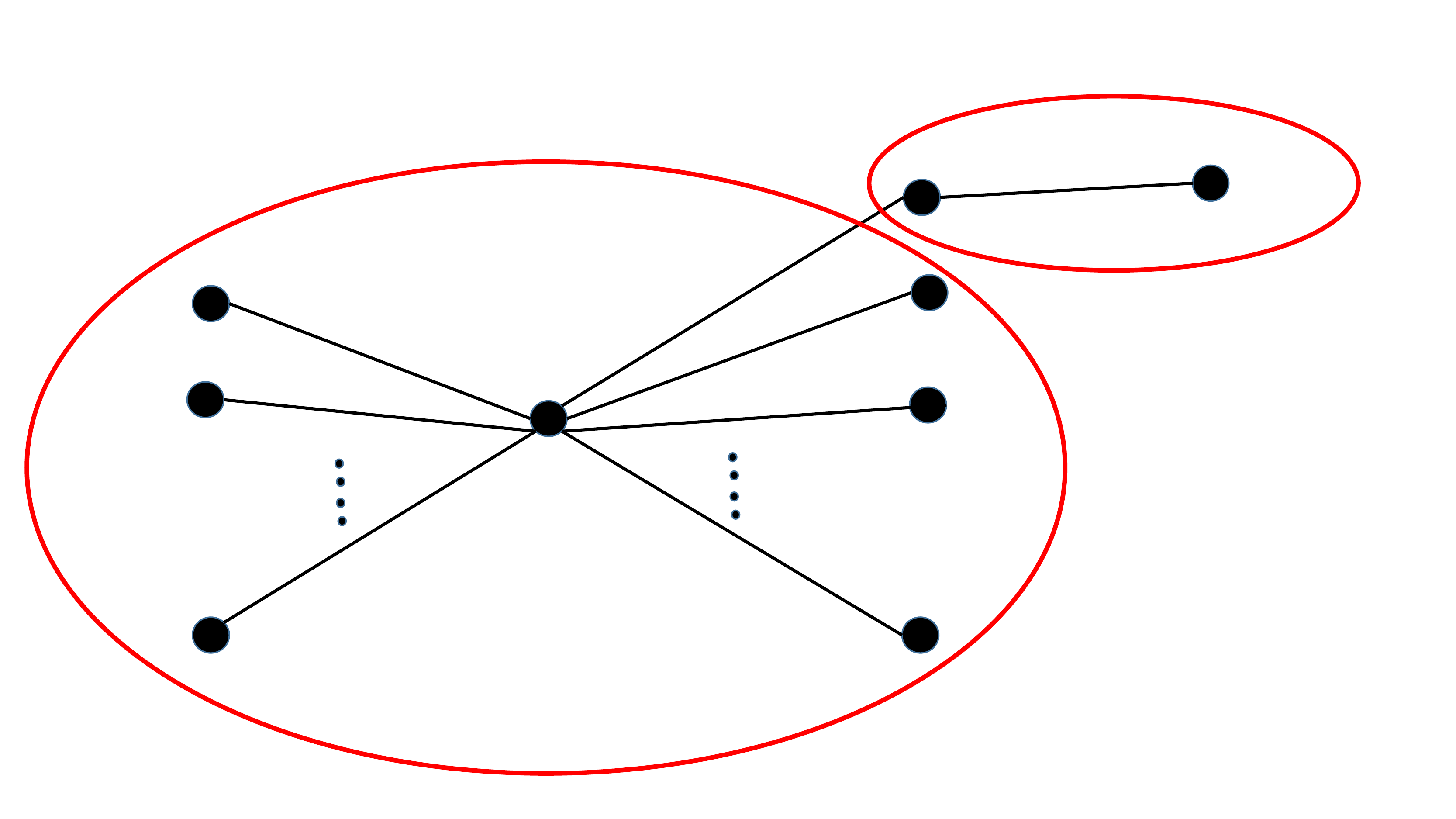}
\put(8,13){$X_t$}
\put(7,29){$X_3$}
\put(15,37){$X_1$}
\put(86,43){$X_2$}
\put(63,15){$X_{2t}$}
\put(62,23){$X_{t+3}$}
\put(65,33){$X_{t+2}$}
\put(54,48){$X_{t+1}$}
\put(34,33){$X_{2t+1}$}
\put(73,45){$\rho_{\min}$}
\put(25,33){$\rho_{\min}$}
\put(25,16){$\rho_{\min}$}
\put(45,38){$\rho_{\max}$}
\put(45,16){$\rho_{\max}$}
\end{overpic} \\
 Tree Structure $\T_2$ ($k=2\;\Rightarrow k_{\mathrm{a}}=1,k_{\mathrm{b}}=2$)  &  Tree Structure $\T_{t+1}$ ($k=t+1\;\Rightarrow k_{\mathrm{a}}=2,k_{\mathrm{b}}=1$) \\
 \begin{overpic}[width=.5\textwidth]{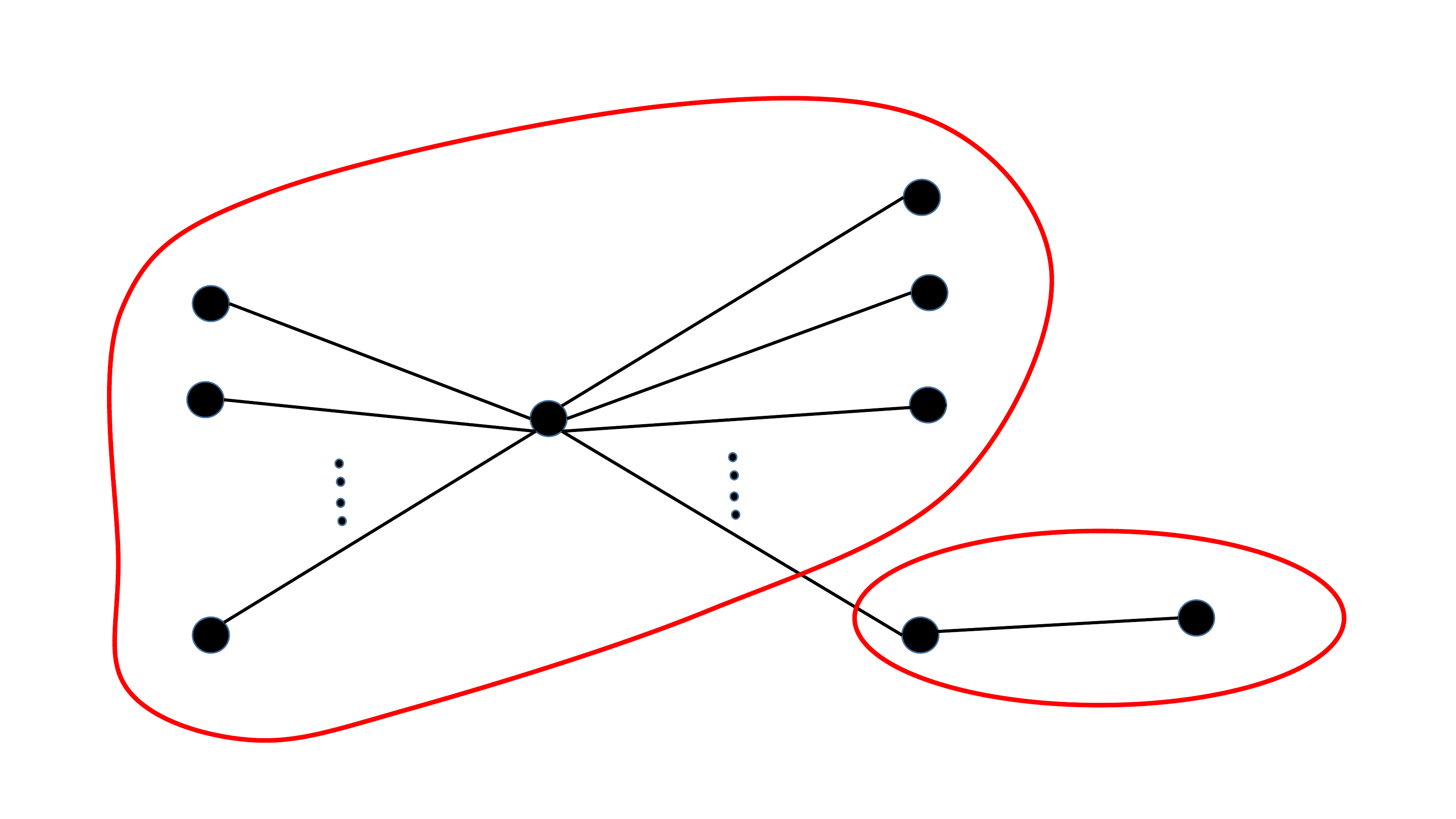}
\put(8,13){$X_t$}
\put(7.5,29){$X_3$}
\put(15,37){$X_1$}
\put(83,11){$X_2$}
\put(61,6){$X_{2t}$}
\put(56,24){$X_{t+3}$}
\put(65,33){$X_{t+2}$}
\put(52,43){$X_{t+1}$}
\put(34,33){$X_{2t+1}$}
\put(73,10){$\rho_{\min}$}
\put(25,33){$\rho_{\min}$}
\put(25,16){$\rho_{\min}$}
\put(45,38){$\rho_{\max}$}
\put(43,17){$\rho_{\max}$}
\end{overpic} 
 &
 \begin{overpic}[width=.5\textwidth]{tree_figures-5.pdf}
\put(11,8){$X_{t-1}$}
\put(7.5,29){$X_2$}
\put(15,37){$X_1$}
\put(83,11){$X_t$}
\put(61,6){$X_{2t}$}
\put(56,24){$X_{t+3}$}
\put(65,33){$X_{t+2}$}
\put(52,43){$X_{t+1}$}
\put(34,33){$X_{2t+1}$}
\put(73,10){$\rho_{\min}$}
\put(25,33){$\rho_{\min}$}
\put(25,16){$\rho_{\min}$}
\put(45,38){$\rho_{\max}$}
\put(43,17){$\rho_{\max}$}
\end{overpic} \\
 Tree Structure $\T_{2t}$ ($k=2t\;\Rightarrow k_{\mathrm{a}}=2,k_{\mathrm{b}}=t$)  &  Tree Structure $\T_{t^2}$ ($k=t^2\;\Rightarrow k_{\mathrm{a}}=t,k_{\mathrm{b}}=t$)  
\end{tabular}
\caption{Tree structures constructed in the proof of Theorem~\ref{thm:NecessarySamples}. Equivalence clusters are circled in red, where an equivalence cluster is defined as a set containing a non-leaf (internal) node and all the leaf nodes connected to it~\citep{Katiyar20arxiv}.} 
\label{fig:converse}
\end{figure} 

We assume, for simplicity, that $d$ is odd, with $d=2t+1$ for some  $t\in\mathbb{N}$. Let the edge set $\cE_0$ consist of $2t$ edges $\big\{ \{X_j, X_{2t+1}\} \big\}_{j=1}^{2t}$, and let the corresponding tree (resp.\ distribution) be denoted $\T_0$ (resp.\ $P_0$). From \eqref{eq:TreeFactorizationP_v2}  and \eqref{eq:ThetaAndRho}, we know that a tree distribution is uniquely defined by the edge correlation values. Let the edge correlations, under distribution $P_0$, be given by
\begin{equation} \label{eq:EdgeCorrelationsE0}
	\rho_{j,2t+1} = \begin{cases}
		\rho_{\min}, ~\mathrm{if~} 1 \le j \le t \\
		\rho_{\max}, ~\mathrm{if~} t+1 \le j \le 2t .
	\end{cases}
\end{equation}
Let $k$ be a positive integer satisfying $1 \le k \le t^2$, and  define
$$ k_{\mathrm{a}} \triangleq 1 + \floor{\frac{k-1}{t}}, \quad k_{\mathrm{b}} \triangleq k - (k_{\mathrm{a}}-1)t .$$
It is seen that $1 \le k_{\mathrm{a}}, k_{\mathrm{b}} \le t$, and the pair $(k_{\mathrm{a}},k_{\mathrm{b}})$ is unique for every $k$. Let $\cE_k$ denote the edge set of tree structure $\T_k$, where
\begin{equation*}
	\cE_k = \{k_{\mathrm{a}}, k_{\mathrm{b}}+t\} \cup \cE_0 \setminus \{k_{\mathrm{a}}, 2t + 1\}.
\end{equation*}
Hence, for $1 \le k \le t^2$, the edge set $\cE_k$ differs from $\cE_0$ in only one edge. 
Let $P_k$ denote the tree distribution corresponding to $\T_k$. For $\{k_{\mathrm{a}}, k_{\mathrm{b}}+t\} \in \cE_k$, let $\rho_{k_{\mathrm{a}},k_{\mathrm{b}}+t} = \rho_{\min}$, and let the edge correlation for the remaining edges in $\cE_k$ be given by~\eqref{eq:EdgeCorrelationsE0}. 

Now consider the noise model in Sec.~\ref{sec:SystemModel}, with $q_i$ denoting the crossover probability for a sample corresponding to the $i$th node. For $d=2t+1$ nodes, we fix these values as follows
\begin{equation} \label{eq:NoiseDist}
	q_i = \begin{cases}
		q_{\max}, ~&\mbox{if~} 1 \le i \le t \\
		0, ~&\mbox{else}.
	\end{cases}
\end{equation}
Let the above parameters be applicable to the noisy samples obtained from tree structure $\T_k$, where $0 \le k \le t^2$. See Fig.~\ref{fig:converse} for diagrams of the trees constructed.

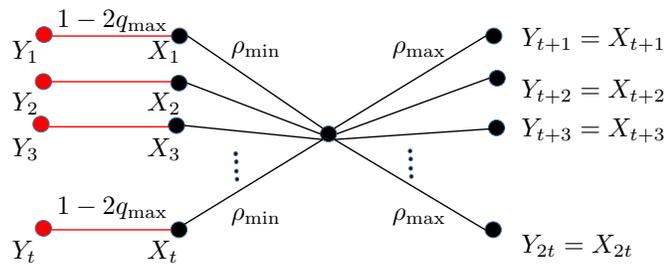
\begin{figure}
\centering
\begin{tabular}{c}
\begin{overpic}[width=.5\textwidth]{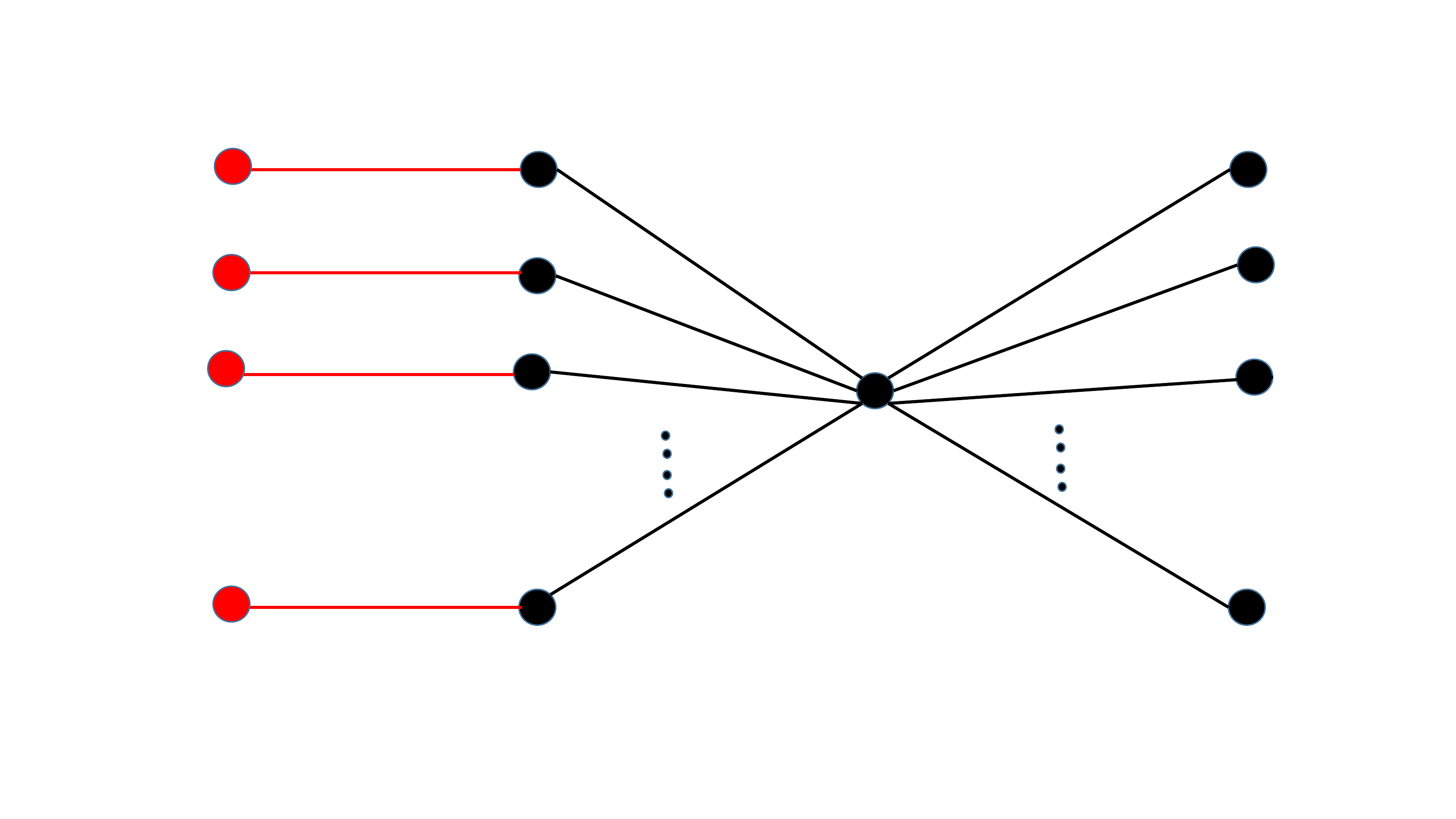}
\put(32,10){$X_t$}
\put(32,26){$X_3$}
\put(32,33){$X_2$}
\put(32,41){$X_1$}

\put(11,10){$Y_t$}
\put(11,26){$Y_3$}
\put(11,33){$Y_2$}
\put(11,41){$Y_1$}

\put(18,46){$1-2q_{\max}$}
\put(18,17){$1-2q_{\max}$}

\put(90,11){$Y_{2t}=X_{2t}$}
\put(90,29){$Y_{ t+3}=X_{t+3}$}
\put(90,35){$Y_{ t+2}=X_{t+2}$}
\put(90,43){$Y_{ t+1}=X_{t+1}$}

\put(45,42){$\rho_{\min}$}
\put(45,16){$\rho_{\min}$}
\put(70,42){$\rho_{\max}$}
\put(70,16){$\rho_{\max}$}
\end{overpic} \vspace{-.2in}
\end{tabular}
\caption{Construction of the distribution of the noisy samples $\{ Y_i\}$.}
\label{fig:noisy}
\end{figure}
 Let $Y_i$ denote the noisy sample corresponding to the $i$th node. Then, we have $\tilde{\rho}_{i,j} \triangleq \bbE[Y_i Y_j] = (1-2q_i)(1-2q_j) \rho_{i,j}$~\cite{Nikolakakis19NonParametric}.
Let $\tilde{P}_k$ denote the distribution for the noisy vector $(Y_1, Y_2, \ldots , Y_{2t+1})$. As noise is only applied to leaf nodes~\eqref{eq:NoiseDist}, we have $Y_j = X_j$ for $t+1 \le j \le 2t+1$, and the conditional independence among $Y_i$, $1 \le i \le 2t+1$ continues to remain encoded via the tree structure $\T_k$.\footnote{In general, the noisy samples do not satisfy the Ising model~\cite{Bresler13,Nikolakakis19AISTATS}. However, noisy samples from an underlying tree-structured Ising model retain the tree-structure when noise is only applied to leaf nodes.}   See Fig.~\ref{fig:noisy} for the construction of the distribution of the noisy samples.

Let $M = t^2$, and let the tree structure $\T$ be chosen uniformly from the set $\{\T_0, \T_1, \ldots , \T_M\}$. Then we have the Markov chain $\T \longrightarrow \brX_1^n \longrightarrow \brY_1^n \longrightarrow \hat{\T}$, and   Fano's inequality gives a lower bound on the error probability $\bbP\big(\Psi(\brY_1^n) \neq \T\big)$ for any estimator $\Psi$ using the multiple hypothesis testing framework. A key observation from our construction of the $M+1$ tree structures $\T_k$, $0 \le k \le M$, is that their corresponding equivalence classes (see Sec.~\ref{sec:EquivClass}) are disjoint, i.e., $[\T_i] \cap [\T_j] = \varnothing$ for $i \neq j$. When learning the underlying tree structure using the multiple hypothesis framework, this observation has the important consequence that $\bbP\big(\Psi(\brY_1^n) \notin [\T]\big) = \bbP\big(\Psi(\brY_1^n) \neq \T\big)$, and we have the following result.
\begin{lemma}[Fano's Inequality, Lemma~6.2 in \citep{BreslerKarzand18}] \label{lem:Fano}
	For $k \in \{0,1,\ldots,M\}$, let $\tilde{P}_k$ be the probability law of the noisy observation $\brY$, with the models satisfying the properties given in Sec.~\ref{sec:SystemModel}. Let $\Psi : \{+1,-1\}^{d \times n} \to \cT_d$ denote an estimator using $n$ i.i.d.\  samples $\brY_1^n$. Let the KL-divergence be $D(\tilde{P}_k \| \tilde{P}_0) \triangleq \sum_{\bry \in \{+1,-1\}^{d}} \tilde{P}_k(\bry) \log\big(\tilde{P}_k(\bry) / \tilde{P}_0(\bry) \big)$, and define the symmetric KL-divergence $J(\tilde{P}_k, \tilde{P}_0) \triangleq D(\tilde{P}_k \| \tilde{P}_0) + D(\tilde{P}_0 \| \tilde{P}_k)$. If the number of samples satisfy
	\begin{equation*}
		n \,<\, (1-\delta) \frac{\log M}{\frac{1}{M+1} \sum_{k=1}^M J(\tilde{P}_k ,\tilde{P}_0)} ,
	\end{equation*} 
	then the minimax error $\cM_n(q_{\max}, \rho_{\min}, \rho_{\max})$ in~\eqref{eq:Def_MinimaxError} is lower bounded as
	\begin{equation*}
		\cM_n(q_{\max}, \rho_{\min}, \rho_{\max}) \ge \delta - \frac{1}{\log M} .
	\end{equation*}
\end{lemma}
We remark that Lemma~\ref{lem:Fano} continues to hold if the symmetric KL-divergence $J(\tilde{P}_k, \tilde{P}_0)$ is replaced by $D(\tilde{P}_k \| \tilde{P}_0)$; however, we use the given form as it is easier to quantify $J(\tilde{P}_k,\tilde{P}_0)$ for Ising models. We proceed to prove Theorem~\ref{thm:NecessarySamples} by quantifying $J(\tilde{P}_k,\tilde{P}_0)$ and applying Lemma~\ref{lem:Fano}.

As discussed previously, if $Y_j$ denotes a noisy sample corresponding to the $j$th node, then the conditional independence among $Y_j$, $1 \le j \le 2t+1$, continues to be encoded by the tree structure $\T_k$ for distribution $\tilde{P}_k$. Further, for $1 \le k \le M$, the edge set $\cE_k$ differs from $\cE_0$ in only one edge, and we have $\cE_k \setminus \cE_0 = \{k_{\mathrm{a}}, k_{\mathrm{b}} + t\}$ and $\cE_0 \setminus \cE_k = \{k_{\mathrm{a}}, 2t+1\}$. Let $\rho_{j_1,j_2}^{(k)}$ (resp. $\rho_{j_1,j_2}^{(0)}$) denote the correlation $\bbE[X_{j_1} X_{j_2}]$ with respect to the distribution $P_k$ (resp. $P_0$). Then, from the construction of $P_k$ and $P_0$ we have
\begin{align}
	\rho_{k_{\mathrm{a}},k_{\mathrm{b}}+t}^{(k)} &= \rho_{\min}, \quad \rho_{k_{\mathrm{b}}+t,2t+1}^{(k)} = \rho_{\max} , \label{eq:rho_k}\\
	\rho_{k_{\mathrm{a}},2t+1}^{(0)} &= \rho_{\min}, \quad \rho_{k_{\mathrm{b}}+t,2t+1}^{(0)} = \rho_{\max} . \label{eq:rho_0}
\end{align}

Let $\tilde{\rho}_{j_1,j_2}^{(k)}$ (resp. $\tilde{\rho}_{j_1,j_2}^{(0)}$) denote the correlation $\bbE[Y_{j_1} Y_{j_2}]$ with respect to the distribution $\tilde{P}_k$ (resp. $\tilde{P}_0$). Then, using \eqref{eq:NoiseDist}, \eqref{eq:rho_k},  and~\eqref{eq:rho_0}, and applying the correlation decay property for tree-structured Ising models~\citep[Lemma~A.2]{Nikolakakis19Predictive}, we obtain
\begin{align}
	\tilde{\rho}_{k_{\mathrm{a}},k_{\mathrm{b}}+t}^{(k)} &= (1-q_{\max})\rho_{\min}, \quad\quad~~~~ \tilde{\rho}_{k_{\mathrm{a}},2t+1}^{(k)} = (1-q_{\max})\rho_{\min} \rho_{\max}, \label{eq:tilde_rho_k}\\
	\tilde{\rho}_{k_{\mathrm{a}},k_{\mathrm{b}}+t}^{(0)} &= (1-q_{\max})\rho_{\min} \rho_{\max}, \quad \tilde{\rho}_{k_{\mathrm{a}},2t+1}^{(0)} = (1-q_{\max}) \rho_{\min} . \label{eq:tilde_rho_0}
\end{align}
Finally, using \eqref{eq:TreeFactorizationP_v2}, \eqref{eq:ThetaAndRho}, \eqref{eq:tilde_rho_k}, \eqref{eq:tilde_rho_0}, and Eqn.~(6.3) in \citet{BreslerKarzand18}, we obtain
\begin{equation}
	J(\tilde{P}_k,\tilde{P}_0) \,=\, 2\, \atanh(\rho_q) \,\rho_q \left(1-\rho_{\max}\right) , \label{eq:SymmKL}
\end{equation}
where $\rho_q = (1-q_{\max})\rho_{\min}$. Now, using \eqref{eq:SymmKL} and Lemma~\ref{lem:Fano}, we observe that if the number of samples satisfy
\begin{equation*}
	n \,<\, (1-\delta) \frac{\log M}{2\, \atanh(\rho_q) \,\rho_q \left(1-\rho_{\max}\right)} ,
\end{equation*} 
then we have $\cM_n(q_{\max}, \rho_{\min}, \rho_{\max}) \ge \delta - (1/ \log M)$. Setting $\delta = 1/2 + (1/ \log M)$, and using the fact $M = t^2$, we get that $\cM_n(q_{\max}, \rho_{\min}, \rho_{\max}) \ge 1/2$ if $n$ satisfies
\begin{equation}
	n \,<\,  \frac{\log(t) - 1}{2\, \atanh(\rho_q) \,\rho_q \left(1-\rho_{\max}\right)} . \label{eq:NecessarySamples2}
\end{equation}
The proof of Theorem~\ref{thm:NecessarySamples} is complete by using \eqref{eq:NecessarySamples2} and observing that $\log(t)-1 = \log(\frac{d-1}{2}) - 1 > \frac{1}{2} \log d$ for $d>32$.
\qed

\section{Minimum and Maximum Possible Size of $[\T]$} \label{app:EqClassSize}

The following proposition quantifies the minimum and maximum size of the equivalence class $[\T]$, for a given number of nodes $d$.
\begin{proposition} \label{prop:EqClassSize}
	Let $d \ge 4$. Then, we have 
	\begin{align}
		\min_{T \in \cT_d}\, |[\T]| &= 4 , \label{eq:MinEqClassSize} \\
		\max_{T \in \cT_d}\, |[\T]| &\le 3^{(d/3)} , \label{eq:MaxEqClassSize}
	\end{align}
	where the minimum is achieved in \eqref{eq:MinEqClassSize} for a chain tree structure, and the inequality in~\eqref{eq:MaxEqClassSize} becomes tight when $d$ is a multiple of $3$.
\end{proposition}
\begin{proof}
	For a given $d \ge 4$, and $\T \in \cT_d$,  recall that  $\cL_{\T}$ is the set of leaf nodes of $\T$. Now, partition $\cL_{\T}$ into smaller subsets, such that all elements in the same subset share a common neighbor in the tree structure $\T$. Let $\kappa$ denote the number of such distinct subsets, and let $\cL_{\T}^{(i)}$ denote the $i$th subset. Then, $\cL_{\T}$ can be expressed as a disjoint union of $\cL_{\T}^{(i)}$ as
	\begin{equation*}
		\cL_{\T} = \meduplus_{i=1}^{\kappa} \, \cL_{\T}^{(i)} .
	\end{equation*}
	Now, if $\ell_i \triangleq |\cL_{\T}^{(i)}|$, then it follows from~\eqref{eq:Def_EquivalenceClassOfT} that the size of the equivalence class is given by
	\begin{equation}
		|\,[\T]\,| = \prod_{i=1}^{\kappa} (1 + \ell_i) . \label{eq:SizeEqClass}
	\end{equation} 
	When $\kappa = 1$, then the tree structure is a \emph{star}~\cite{Tan11IT}, and we have $|\,[\T]\,| = d \ge 4$. As $\ell_i \ge 1$, it follows from~\eqref{eq:SizeEqClass} that for $\kappa \ge 2$, we have $|\,[\T]\,| \ge 4$. In particular, for a chain tree structure, we have $\kappa = 2$ with $\ell_1 = \ell_2 = 1$, and hence $|\,[\T]\,| = 4$ for a chain. 
	
	We proceed to prove the upper bound in~\eqref{eq:MaxEqClassSize}. From~\eqref{eq:SizeEqClass}, it follows that the size $|\,[\T]\,|$ can be upper bounded by the solution of the following constrained maximization problem,
	\begin{equation*}
		|\,[\T]\,| \le \max_{m_1 + \cdots + m_{\kappa} \le d}\ \prod_{i=1}^{\kappa} m_i.
	\end{equation*}
	The expression on the right side is upper bounded by $3^{d/3}$, and this bound is tight when $d \!\mod 3 = 0$~\cite{Krause96}. Thus, we have
	\begin{equation*}
		|\,[\T]\,| \le 3^{d/3} ,
	\end{equation*}
	with equality if $d$ is a multiple of $3$. Note that when $d \mod 3 = 0$, and $\kappa = d/3$ with $\ell_1 = \cdots = \ell_{\kappa} = 2$, then we have $|\,[\T]\,| = 3^{d/3}$.
\end{proof}

\section{Overview of the Algorithm by~\citet{Katiyar20arxiv} for Declaring Star/Non-star} \label{app:KatiyarAlgoOverview}

\begin{algorithm}[tb]
	\caption{IS\_NON\_STAR}
	\label{alg:Kat_StarNonStar}
	\begin{algorithmic}
		\STATE Let the set of $4$ nodes be $\{X_1,X_2,X_3,X_4\}$
		
		\STATE {\bfseries Input:} Empirical correlations $\widehat{\rho}_{i,j}, \ 1 \le i < j \le 4$, \ Threshold $\alpha = \frac{1 + \rho_{\max}^2}{2}$
		
		\IF{$\frac{\widehat{\rho}_{1,3}\,\widehat{\rho}_{2,4}}{\widehat{\rho}_{1,2}\,\widehat{\rho}_{3,4}} < \alpha$ and $\frac{\widehat{\rho}_{1,3}\,\widehat{\rho}_{2,4}}{\widehat{\rho}_{1,4}\,\widehat{\rho}_{2,3}} > \alpha$}
		
		\STATE Declare Non-star where $\{X_1,X_2\}$ forms a pair
		
		\ELSIF{$\frac{\widehat{\rho}_{1,2}\,\widehat{\rho}_{3,4}}{\widehat{\rho}_{1,3}\,\widehat{\rho}_{2,4}} < \alpha$ and $\frac{\widehat{\rho}_{1,2}\,\widehat{\rho}_{3,4}}{\widehat{\rho}_{1,4}\,\widehat{\rho}_{2,3}} > \alpha$}
		
		\STATE Declare Non-star where $\{X_1,X_3\}$ forms a pair
		
		\ELSIF{$\frac{\widehat{\rho}_{1,2}\,\widehat{\rho}_{3,4}}{\widehat{\rho}_{1,4}\,\widehat{\rho}_{2,3}} < \alpha$ and $\frac{\widehat{\rho}_{1,2}\,\widehat{\rho}_{3,4}}{\widehat{\rho}_{1,3}\,\widehat{\rho}_{2,4}} > \alpha$}
		
		\STATE Declare Non-star where $\{X_1,X_4\}$ forms a pair
		
		\ELSE
		
		\STATE Declare Star
		
		\ENDIF
		
	\end{algorithmic}
\end{algorithm}

Let $\bry_1^n = \{\bry_1, \ldots , \bry_n\}$ denote $n$ independently sampled noisy observations, where the $k$th noisy sample is a $d$-dimensional column vector $\bry_k = (y_{k,1}, \ldots , y_{k,d})^T$. The estimator in~\cite{Katiyar20arxiv} proceeds by first calculating the pairwise empirical correlations,
\begin{equation}
	\widehat{\rho}_{i,j} \triangleq \frac{1}{n} \sum_{k=1}^n y_{k,i} \, y_{k,j}, \label{def:EmpiricalCorr}
\end{equation}
where $1 \le i < j \le d$. The procedure used in~\cite{Katiyar20arxiv} to declare a set of $4$ nodes as star or non-star, based on the knowledge of empirical correlations $\widehat{\rho}_{i,j}$, is described in Algorithm~\ref{alg:Kat_StarNonStar}. The intuition behind Algorithm~\ref{alg:Kat_StarNonStar} can be roughly outlined by considering an example where the $4$ nodes form a Markov-chain $X_1 \myrule X_2 \myrule X_3 \myrule X_4$. If the noisy correlations are denoted $\tilde{\rho}_{i,j} \triangleq \bbE[Y_i Y_j]$,\footnote{Note that $\tilde{\rho}_{i,j} = (1-2q_i)(1-2q_j)\, \rho_{i,j}$.} then we have $\frac{\tilde{\rho}_{1,3}\,\tilde{\rho}_{2,4}}{\tilde{\rho}_{1,2}\,\tilde{\rho}_{3,4}} \le \rho_{\max}^2$ and $\frac{\tilde{\rho}_{1,3}\,\tilde{\rho}_{2,4}}{\tilde{\rho}_{1,4}\,\tilde{\rho}_{2,3}} = 1$, and hence we would expect the empirical correlations to satisfy the conditions $\frac{\widehat{\rho}_{1,3}\,\widehat{\rho}_{2,4}}{\widehat{\rho}_{1,2}\,\widehat{\rho}_{3,4}} < \alpha$ and $\frac{\widehat{\rho}_{1,3}\,\widehat{\rho}_{2,4}}{\widehat{\rho}_{1,4}\,\widehat{\rho}_{2,3}} > \alpha$, where $\alpha = (1 + \rho_{\max}^2)/{2}$.

The partial tree structure  learning algorithm detailed in~\citet{Katiyar20arxiv} ensures that if Algorithm~\ref{alg:Kat_StarNonStar} correctly declares any set of $4$ nodes as star or non-star (with appropriate pairing of nodes), then the equivalence class $[\T]$ is successfully detected, i.e. $\Psi(\bry_1^n) \in [\T]$. Therefore, the performance of the estimator critically depends on the accuracy of Algorithm~\ref{alg:Kat_StarNonStar}.


\section{Proof of Theorem~\ref{thm:ImprovedAchievabilityBound}} \label{app:ImprovedAchievabilityBound}
The algorithm by~\citet{Katiyar20arxiv} correctly estimates the equivalence class $[\T]$ if any set of $4$ nodes within each others \emph{proximal sets} are declared correctly as \emph{star} or \emph{non-star}. The algorithm used for declaring $4$ nodes as star or non-star is described in Alg.~\ref{alg:Kat_StarNonStar} in App.~\ref{app:KatiyarAlgoOverview}.  Let $\bry_1^n = \{\bry_1, \ldots , \bry_n\}$ denote $n$ independently sampled noisy observations, where the $i$th noisy sample is a $d$-dimensional column vector $\bry_i = (y_{i,1}, \ldots , y_{i,d})^T$. The algorithm proceeds by first calculating the pairwise empirical correlations, $\widehat{\rho}_{j,k} \triangleq \frac{1}{n} \sum_{i=1}^n y_{i,j} \, y_{i,k}$, where $1 \le j < k \le d$.

Without loss of generality, consider the set of $4$ nodes $\{X_{1}, X_{2}, X_{3}, X_{4}\}$ with the corresponding noisy variables $\{Y_{1}, Y_{2}, Y_{3}, Y_{4}\}$, and let $\tilde{\rho}_{j,k} = \mathbb{E}[Y_j\, Y_k]$. Let these nodes form a non-star with $\{X_1, X_2\}$ as a pair. From the procedure in Alg.~\ref{alg:Kat_StarNonStar} in App.~\ref{app:KatiyarAlgoOverview}, it follows that a correct decision is made if 
\begin{equation}
	\frac{\widehat{\rho}_{1,3}\,\widehat{\rho}_{2,4}}{\widehat{\rho}_{1,2}\,\widehat{\rho}_{3,4}} \;<\; \alpha   \quad \mbox{and} \quad \frac{\widehat{\rho}_{1,3}\,\widehat{\rho}_{2,4}}{\widehat{\rho}_{1,4}\,\widehat{\rho}_{2,3}} \;>\; \alpha , \label{eq:KatSuffCondCorrectDecision_v1}
\end{equation}
where $\alpha = (1 + \rho_{\max}^2)/2$. Now, as $\{X_1, X_2\}$ forms a pair, we have
\begin{equation}
	\frac{\tilde{\rho}_{1,3}\,\tilde{\rho}_{2,4}}{\tilde{\rho}_{1,2}\,\tilde{\rho}_{3,4}} \;\le\; \rho_{\max}^2  \quad \mbox{and} \quad \frac{\tilde{\rho}_{1,3}\,\tilde{\rho}_{2,4}}{\tilde{\rho}_{1,4}\,\tilde{\rho}_{2,3}} \;=\; 1 . \label{eq:KatIdealStats_v1}
\end{equation}
Define $\Delta_{j,k} \triangleq \tilde{\rho}_{j,k} - \widehat{\rho}_{j,k}$, and $\Delta = \max_{1 \le j < k \le 4} | \Delta_{j,k} |$. We will show that the inequalities in~\eqref{eq:KatSuffCondCorrectDecision_v1} are satisfied if  
\begin{equation}
	\Delta < \tilde{\delta} \triangleq \frac{t_2(1-\alpha)}{20}, \label{eq:KatSuffCondCorrectDecision_v2}
\end{equation}
where  $t_2 = \min\left\{t_1, \frac{t_1 (1-2q_{\max})}{\rho_{\max}}\right\}$ and $t_1 = (1-2q_{\max})^2 \rho_{\min}^4$. Assume the inequality in \eqref{eq:KatSuffCondCorrectDecision_v2} to be true and define $\beta \triangleq 0.1 (1-\alpha)$. Then, for $1 \le j < k \le 4$, we have 
\begin{equation}
	\left| \frac{\Delta_{j,k}}{\widehat{\rho}_{i,j}} \right| \: \overset{(a)}{\le} \frac{\Delta}{0.5 t_2} \:
	\overset{(b)}{<} \frac{\tilde{\delta}}{0.5 t_2} \:
	\overset{(c)}{=} \beta , \label{eq:Kat_BoundRatio}
\end{equation}
where $(a)$ follows from the fact that proximal sets are chosen to satisfy $|\widehat{\rho}_{i,j}| \ge 0.5 t_2$, \ $(b)$~follows from \eqref{eq:KatSuffCondCorrectDecision_v2}, and $(c)$~follows from the definitions of $\tilde{\delta}$ and $\beta$. Now, we have
\begin{align}
	\rho_{\max}^2 &\overset{(d)}{\ge} 	\frac{\tilde{\rho}_{1,3}\,\tilde{\rho}_{2,4}}{\tilde{\rho}_{1,2}\,\tilde{\rho}_{3,4}} \nonumber \\
	&= \frac{\big(\widehat{\rho}_{1,3} + \Delta_{1,3}\big) \big(\widehat{\rho}_{2,4} + \Delta_{2,4}\big)}{\big(\widehat{\rho}_{1,2} + \Delta_{1,2}\big)\big(\widehat{\rho}_{3,4} + \Delta_{3,4}\big)} \nonumber\\
	&= \frac{\widehat{\rho}_{1,3}\,\widehat{\rho}_{2,4}}{\widehat{\rho}_{1,2}\,\widehat{\rho}_{3,4}}\, \frac{\big(1 + \Delta_{1,3}/\widehat{\rho}_{1,3}\big) \big(1 + \Delta_{2,4}/\widehat{\rho}_{2,4}\big)}{\big(1 + \Delta_{1,2}/\widehat{\rho}_{1,2}\big)\big(1 + \Delta_{3,4}/\widehat{\rho}_{3,4}\big)} \nonumber \\
	&\overset{(e)}{>} \frac{\widehat{\rho}_{1,3}\,\widehat{\rho}_{2,4}}{\widehat{\rho}_{1,2}\,\widehat{\rho}_{3,4}}\, \frac{(1-\beta)^2}{(1+\beta)^2} , \label{eq:Kat_Ineq1}
\end{align}
where $(d)$ follows from~\eqref{eq:KatIdealStats_v1}, and $(e)$ follows from~\eqref{eq:Kat_BoundRatio}. We can equivalently express~\eqref{eq:Kat_Ineq1} as
\begin{align}
	\frac{\widehat{\rho}_{1,3}\,\widehat{\rho}_{2,4}}{\widehat{\rho}_{1,2}\,\widehat{\rho}_{3,4}} &< \rho_{\max}^2 \frac{(1+\beta)^2}{(1-\beta)^2} \nonumber \\
	&< \rho_{\max}^2 \frac{1+2.1\beta}{1-2\beta} \nonumber \\
	&< \rho_{\max}^2 (1+2.1\beta) (1+3\beta) \nonumber \\
	&< \rho_{\max}^2 (1+6\beta), \label{eq:Kat_Ineq2}
\end{align} 
where we have applied the fact that $\beta < 0.1$, and hence $\beta^2 < 0.1 \beta$. As $\beta = 0.1 (1-\alpha) = (1-\rho_{\max}^2)/20$, it follows from~\eqref{eq:Kat_Ineq2} that 
\begin{align}
	\frac{\widehat{\rho}_{1,3}\,\widehat{\rho}_{2,4}}{\widehat{\rho}_{1,2}\,\widehat{\rho}_{3,4}} \ &< \rho_{\max}^2 + 0.3\rho_{\max}^2(1-\rho_{\max}^2) \nonumber \\
	&<  \rho_{\max}^2 + 0.3(1-\rho_{\max}^2) \nonumber \\
	&= 0.3 + 0.7 \rho_{\max}^2 \nonumber \\
	&< 0.5 + 0.5 \rho_{\max}^2 \nonumber \\
	&= \alpha , \nonumber
\end{align}
and this proves the first inequality in~\eqref{eq:KatSuffCondCorrectDecision_v1}. To prove the second  inequality in~\eqref{eq:KatSuffCondCorrectDecision_v1}, we note that
\begin{align}
	1 &= \frac{\tilde{\rho}_{1,3}\,\tilde{\rho}_{2,4}}{\tilde{\rho}_{1,4}\,\tilde{\rho}_{2,3}} \nonumber \\
	&= \frac{\widehat{\rho}_{1,3}\,\widehat{\rho}_{2,4}}{\widehat{\rho}_{1,4}\,\widehat{\rho}_{2,3}}\, \frac{\big(1 + \Delta_{1,3}/\widehat{\rho}_{1,3}\big) \big(1 + \Delta_{2,4}/\widehat{\rho}_{2,4}\big)}{\big(1 + \Delta_{1,4}/\widehat{\rho}_{1,4}\big)\big(1 + \Delta_{2,3}/\widehat{\rho}_{2,3}\big)} \nonumber \\
	&< \frac{\widehat{\rho}_{1,3}\,\widehat{\rho}_{2,4}}{\widehat{\rho}_{1,4}\,\widehat{\rho}_{2,3}}\, \frac{(1+\beta)^2}{(1-\beta)^2} , \label{eq:Kat_Ineq3}
\end{align}
where the last inequality follows from~\eqref{eq:Kat_BoundRatio}. We can equivalently express~\eqref{eq:Kat_Ineq3} as 
\begin{align}
	\frac{\widehat{\rho}_{1,3}\,\widehat{\rho}_{2,4}}{\widehat{\rho}_{1,4}\,\widehat{\rho}_{2,3}}\, &> \frac{(1-\beta)^2}{(1+\beta)^2} 
	\nonumber \\
	&> \frac{1-2\beta}{1+2.1\beta} \nonumber \\
	&> (1-2\beta)(1-2.1\beta) \nonumber \\
	&> 1 - 5 \beta \nonumber \\
	&= 1 - 0.25(1- \rho_{\max}^2) \nonumber \\
	&= 0.75 + 0.25 \rho_{\max}^2  \nonumber \\
	&> 0.5 + 0.5 \rho_{\max}^2 \nonumber \\
	&= \alpha , \nonumber
\end{align}
thereby proving the second inequality in~\eqref{eq:KatSuffCondCorrectDecision_v1}. Thus, we have shown that if  $\{X_{1}, X_{2}, X_{3}, X_{4}\}$ form a non-star with pair $\{X_1,X_2\}$, then the condition in~\eqref{eq:KatSuffCondCorrectDecision_v2} is sufficient for the algorithm to make a correct decision. In a similar fashion, it can be shown that \eqref{eq:KatSuffCondCorrectDecision_v2} provides a sufficient condition for making the correct decision even when $\{X_{1}, X_{2}, X_{3}, X_{4}\}$ form a star or a non-star with a different pairing.

Define the event $\cB_{j,k}$ for $1 \le j < k \le d$ as
\begin{equation}
	\cB_{j,k} \triangleq \left\{ \left|\widehat{\rho}_{j,k} - \tilde{\rho}_{j,k} \right| \ge \tilde{\delta} \right\} . \label{eq:Kat_ErrorEvent_jk}
\end{equation}
Then, as \eqref{eq:KatSuffCondCorrectDecision_v2} is a sufficient condition for correct declaration as star/non-star for any set $4$ nodes that are within the proximal sets of each other, it follows that for any $P \in \cP_{\T}(\rho_{\min}, \rho_{\max})$, the error probability $\bbP_P\big(\Psi(\brY_1^n) \notin [\T] \big)$ can be upper bounded as
\begin{equation}
	\bbP_P\big(\Psi(\brY_1^n) \notin [\T] \big) \,\le\, \Pr\big( \medcup_{1 \le j < k \le d} \cB_{j,k} \big) . \label{eq:Kat_OverallErrorProb_v1}
\end{equation}
From the definition of the event $\cB_{j,k}$ in~\eqref{eq:Kat_ErrorEvent_jk}, it follows using Hoeffding's inequality that
\begin{equation}
	\Pr\big(\cB_{j,k}\big) \, \le \, 2 \exp\left(- \frac{ n\, {\tilde{\delta}}^{2}}{2} \right) . \label{eq:Kat_ErrorProb_jk}
\end{equation}
Now, using \eqref{eq:Kat_OverallErrorProb_v1}, \eqref{eq:Kat_ErrorProb_jk}, and applying the union bound over $\binom{d}{2}$ pairs of nodes, we obtain
\begin{equation}
	\bbP_P\big(\Psi(\brY_1^n) \notin [\T] \big) \,\le\, d^2 \exp\left( -\frac{ n\, {\tilde{\delta}}^{2}}{2} \right) . \label{eq:Kat_OverallErrorProb_v2} 
\end{equation}
Therefore, for the error probability to be upper bounded by $\tau$, it is sufficient for the number of samples $n$ to satisfy
\begin{equation}
	n \ge \frac{2}{\tilde{\delta}^2} \log \left(\frac{d^2}{\tau} \right). \label{eq:Kat_ImprovedSuffSampleBound}	
\end{equation}
This completes the proof.
\qed

We note that this is much improved over \citet[Theorem 3]{Katiyar20arxiv} as the right-hand-side is $O(1/\tilde{\delta}^2)$ instead of $O(1/\delta^2)$ (see Theorem~\ref{thm:Katiyar_AchievabilityResult}).  Recall that $\tilde{\delta}=\Theta(t_2)$ while $\delta=\Theta(t_2^3)$  and     $t_2 =\min\big\{t_1, \frac{t_1 (1-2q_{\max})}{\rho_{\max}}\big\}$ and $t_1 =(1-2q_{\max})^2 \rho_{\min}^4$.

\section{Proof of Proposition~\ref{prop:ImprovedAchievabilityBound_v2}} \label{app:ImprovedAchievabilityBound_v2}
The proof is similar to the proof of Theorem~\ref{thm:ImprovedAchievabilityBound} in App.~\ref{app:ImprovedAchievabilityBound}, and we focus here only on the important steps. Consider the set of $4$ nodes $\{X_{1}, X_{2}, X_{3}, X_{4}\}$ that forms a non-star with $\{X_1, X_2\}$ as a pair. From the procedure in Alg.~\ref{alg:Mod_StarNonStar}, it follows that a correct decision is made if 
\begin{equation}
	\frac{\sqrt{|\widehat{\rho}_{1,3}\,\widehat{\rho}_{2,4}\, \widehat{\rho}_{1,4}\,\widehat{\rho}_{2,3}|}}{|\widehat{\rho}_{1,2}\,\widehat{\rho}_{3,4}|} \,<\, \alpha, \;\;\; \frac{\widehat{\rho}_{1,3}\,\widehat{\rho}_{2,4}}{\widehat{\rho}_{1,2}\,\widehat{\rho}_{3,4}} \,<\, 1, \;\mbox{and}\;\; \frac{\widehat{\rho}_{1,4}\,\widehat{\rho}_{2,3}}{\widehat{\rho}_{1,2}\,\widehat{\rho}_{3,4}} \,<\, 1 , \label{eq:ModSuffCondCorrectDecision_v1}
\end{equation}
where $\alpha = (1 + \rho_{\max}^2)/2$. Now, as $\{X_1, X_2\}$ forms a pair, we have
\begin{equation}
	\frac{\sqrt{\tilde{\rho}_{1,3}\,\tilde{\rho}_{2,4}\, \tilde{\rho}_{1,4}\,\tilde{\rho}_{2,3}}}{\tilde{\rho}_{1,2}\,\tilde{\rho}_{3,4}} \,\le\, \rho_{\max}^2, \;\;\; \frac{\tilde{\rho}_{1,3}\,\tilde{\rho}_{2,4}}{\tilde{\rho}_{1,2}\,\tilde{\rho}_{3,4}} \,\le\, \rho_{\max}^2, \;\;\mbox{and}\;\; \frac{\widehat{\rho}_{1,4}\,\widehat{\rho}_{2,3}}{\widehat{\rho}_{1,2}\,\widehat{\rho}_{3,4}} \,\le\, \rho_{\max}^2 . \label{eq:ModIdealStats_v1}
\end{equation}
Define $\Delta_{j,k} \triangleq \tilde{\rho}_{j,k} - \widehat{\rho}_{j,k}$, and $\Delta = \max_{1 \le j < k \le 4} | \Delta_{j,k} |$. We will show that the inequalities in~\eqref{eq:ModSuffCondCorrectDecision_v1} are satisfied if  
\begin{equation}
	\Delta < \tilde{\delta} \triangleq \frac{t_2(1-\alpha)}{20}, \label{eq:ModSuffCondCorrectDecision_v2}
\end{equation}
where  $t_2 = \min\left\{t_1, \frac{t_1 (1-2q_{\max})}{\rho_{\max}}\right\}$ and $t_1 = (1-2q_{\max})^2 \rho_{\min}^4$. Assume the inequality in \eqref{eq:ModSuffCondCorrectDecision_v2} to be true and define $\beta \triangleq 0.1 (1-\alpha)$. Then, for $1 \le j < k \le 4$, we have 
\begin{equation}
	\left| \frac{\Delta_{j,k}}{\widehat{\rho}_{i,j}} \right| \: \overset{(a)}{\le} \frac{\Delta}{0.5 t_2} \:
	\overset{(b)}{<} \frac{\tilde{\delta}}{0.5 t_2} \:
	\overset{(c)}{=} \beta , \label{eq:Mod_BoundRatio}
\end{equation}
where $(a)$ follows from the fact that proximal sets are chosen to satisfy $|\widehat{\rho}_{i,j}| \ge 0.5 t_2$, \ $(b)$~follows from \eqref{eq:ModSuffCondCorrectDecision_v2}, and $(c)$~follows from the definitions of $\tilde{\delta}$ and $\beta$. Now, we have
\begin{align}
	\rho_{\max}^2 &\overset{(d)}{\ge} 		\frac{\sqrt{\tilde{\rho}_{1,3}\,\tilde{\rho}_{2,4}\, \tilde{\rho}_{1,4}\,\tilde{\rho}_{2,3}}}{\tilde{\rho}_{1,2}\,\tilde{\rho}_{3,4}} \nonumber \\
	&= \frac{\sqrt{\big(\widehat{\rho}_{1,3} + \Delta_{1,3}\big)\big(\widehat{\rho}_{2,4}+\Delta_{2,4}\big) \big(\widehat{\rho}_{1,4}+\Delta_{1,4}\big)\big(\widehat{\rho}_{2,3}+\Delta_{2,3}\big)}}{\big(\widehat{\rho}_{1,2}+\Delta_{1,2}\big)\big(\widehat{\rho}_{3,4}+\Delta_{3,4}\big)} \nonumber \\
	&= \frac{\sqrt{\widehat{\rho}_{1,3}\, \widehat{\rho}_{2,4}\, \widehat{\rho}_{1,4}\, \widehat{\rho}_{2,3} }}{\widehat{\rho}_{1,2}\, \widehat{\rho}_{3,4}} 
	\, \frac{\sqrt{\big(1 + \Delta_{1,3}/\widehat{\rho}_{1,3}\big)\big(1+\Delta_{2,4}/\widehat{\rho}_{2,4}\big) \big(1+\Delta_{1,4}/\widehat{\rho}_{1,4}\big)\big(1+\Delta_{2,3}/\widehat{\rho}_{2,3}\big)}}{\big(1+\Delta_{1,2}/\widehat{\rho}_{1,2}\big)\big(1+\Delta_{3,4}/\widehat{\rho}_{3,4}\big)} \nonumber \\
	&\overset{(e)}{>} \frac{\sqrt{\widehat{\rho}_{1,3}\, \widehat{\rho}_{2,4}\, \widehat{\rho}_{1,4}\, \widehat{\rho}_{2,3} }}{\widehat{\rho}_{1,2}\, \widehat{\rho}_{3,4}} 
	\, \frac{(1-\beta)^2}{(1+\beta)^2}, 	\label{eq:Mod_Ineq1}
\end{align}
where $(d)$ follows from~\eqref{eq:ModIdealStats_v1}, and $(e)$ follows from~\eqref{eq:Mod_BoundRatio}. Now, we can equivalently express~\eqref{eq:Mod_Ineq1} as 
\begin{align}
	\frac{\sqrt{\widehat{\rho}_{1,3}\, \widehat{\rho}_{2,4}\, \widehat{\rho}_{1,4}\, \widehat{\rho}_{2,3} }}{\widehat{\rho}_{1,2}\, \widehat{\rho}_{3,4}} &< \rho_{\max}^2 \frac{(1+\beta)^2}{(1-\beta)^2} \nonumber \\
	&\overset{(f)}{<} \alpha , \label{eq:Mod_Ineq2}
\end{align}
where $(f)$ follows by employing relations similar to those used around~\eqref{eq:Kat_Ineq2}, and thereby establishes the first inequality in~\eqref{eq:ModSuffCondCorrectDecision_v1}. The other two inequalities in \eqref{eq:ModSuffCondCorrectDecision_v1} can be readily proved in a similar way. Further, this approach can be repeated to prove that the condition in~\eqref{eq:ModSuffCondCorrectDecision_v2} is sufficient for making a correct decision even when the nodes $\{X_{1}, X_{2}, X_{3}, X_{4}\}$ form a star or a non-star with a different pairing. Finally, the steps in~\eqref{eq:Kat_ErrorEvent_jk}--\eqref{eq:Kat_ImprovedSuffSampleBound} can be repeated to complete the proof.
\qed

\section{Proof of Proposition~\ref{prop:KA_ErrExp}} \label{app:KA_ErrExp}

\begin{itemize}
	\item We first prove the claim given by Proposition~\ref{prop:KA_ErrExp}(a) where $P$ corresponds to a Markov chain. We observe from the chain structure that the $4$ nodes form a non-star with $\{X_1,X_2\}$ forming a pair. It follows from the procedure in Algorithm~\ref{alg:Kat_StarNonStar} (in App.~\ref{app:KatiyarAlgoOverview}) that the two events that lead to error using $\Psi_{\KA}$ are $\sE_1 = \Big\{\frac{\widehat{\rho}_{1,3}\,\widehat{\rho}_{2,4}}{\widehat{\rho}_{1,2}\,\widehat{\rho}_{3,4}} \ge \alpha\Big\}$, and $\sE_2 = \Big\{\frac{\widehat{\rho}_{1,3}\,\widehat{\rho}_{2,4}}{\widehat{\rho}_{1,4}\,\widehat{\rho}_{2,3}} \le \alpha\Big\}$. The exponents corresponding to these error events are defined as $ e_i \triangleq \lim_{n \to \infty} -\frac{1}{n} \log \Pr\big( \sE_i\big), \ i  \in \{1,2\}$. Using Sanov's theorem~\citep[Chap.~11]{CoverBook06}, it follows that these exponents are given by~\eqref{eq:KA_e1} and \eqref{eq:KA_e2}, respectively. Now, we have $\bbP_{\tilde{P}}\big(\Psi(\brY_1^n) \notin [\T] \big) = \bbP_{\tilde{P}}\big(\sE_1 \cup \sE_2\big)$, and hence it follows from the definition in~\eqref{eq:Def_ErrorExp} that $E(\Psi_{\KA},\tilde{P}) = \min\{e_1,e_2\}$.
	
	\item We now prove the claim given by Proposition~\ref{prop:KA_ErrExp}(b) where $P$ corresponds to a star structured tree. When the $4$ nodes form a star structure, then an error is made if only if the procedure in Algorithm~\ref{alg:Kat_StarNonStar} (in App.~\ref{app:KatiyarAlgoOverview}) make any one of the following incorrect declarations:
	\begin{itemize}
		\item Non-star with pair $\{X_1,X_2\}$.
		\item Non-star with pair $\{X_1,X_3\}$.
		\item Non-star with pair $\{X_1,X_4\}$.
	\end{itemize}
	By symmetry of the underlying star structure, the probability of each of the three erroneous declarations is same, and hence it is sufficient to analyze the exponent of the probability that a non-star  with pair $\{X_1,X_2\}$ is incorrectly declared, in order to characterize $E(\Psi_{\KA},\tilde{P})$. Now, it follows from Alg.~\ref{alg:Kat_StarNonStar} that a non-star  with pair $\{X_1,X_2\}$ is declared if $\frac{\widehat{\rho}_{1,3}\,\widehat{\rho}_{2,4}}{\widehat{\rho}_{1,2}\,\widehat{\rho}_{3,4}} < \alpha$, and $\frac{\widehat{\rho}_{1,3}\,\widehat{\rho}_{2,4}}{\widehat{\rho}_{1,4}\,\widehat{\rho}_{2,3}} > \alpha$. Finally, using Sanov's theorem~\citep[Ch.~11]{CoverBook06}, it follows that the exponent of the probability that these conditions are satisfied are given by~\eqref{eq:KA_ErrExpStar}.
\end{itemize}
\qed

\section{Proof of Proposition~\ref{prop:SGA_ErrExp}} \label{app:SGA_ErrExp}

\begin{itemize}
	\item We first prove the claim given by Proposition~\ref{prop:SGA_ErrExp}(a) where $P$ corresponds to a Markov chain. We observe from the chain structure that the $4$ nodes form a non-star with $\{X_1,X_2\}$ forming a pair. It follows from the procedure in Algorithm~\ref{alg:Mod_StarNonStar} that a correct decision is made if (i)~$v_2 < \alpha$, (ii)~$v_2 < v_3$, and (iii)~$v_2 < v_4$. Thus, an incorrect decision is made if any of the following events are true.
\begin{itemize}
	\item Event $\sE_3 \triangleq \Big\{\frac{\sqrt{|\widehat{\rho}_{1,3}\,\widehat{\rho}_{2,4}\, \widehat{\rho}_{1,4}\,\widehat{\rho}_{2,3}|}}{|\widehat{\rho}_{1,2}\,\widehat{\rho}_{3,4}|} \ge \alpha\Big\}$, implying $v_2 \ge \alpha$.
	\item Event $\sE_4 \triangleq \left\{|\widehat{\rho}_{1,3}\,\widehat{\rho}_{2,4}| \ge |\widehat{\rho}_{1,2}\,\widehat{\rho}_{3,4}|\right\}$, implying $v_2 \ge v_3$.\footnote{Note that we have taken a slightly pessimistic approach where an error is declared in case of a tie $v_2 = v_3$. In practice, the ties can be broken by a coin toss, but this does not affect the error exponent.}
	\item Event $\sE_5 \triangleq \left\{|\widehat{\rho}_{1,4}\,\widehat{\rho}_{2,3}| \ge |\widehat{\rho}_{1,2}\,\widehat{\rho}_{3,4}|\right\}$, implying $v_2 \ge v_4$.	
\end{itemize}

The exponents corresponding to these error events are defined as $ e_i \triangleq \lim_{n \to \infty} -\frac{1}{n} \log \Pr\big( \sE_i\big), \ i = \{3,4,5\}$. Using Sanov's theorem~\citep[Chap.~11]{CoverBook06}, it follows that these exponents are given by~\eqref{eq:SGA_e3}, \eqref{eq:SGA_e4} and \eqref{eq:SGA_e5}, respectively. Now, we have $\bbP_{\tilde{P}}\big(\Psi(\brY_1^n) \notin [\T] \big) = \bbP_{\tilde{P}}\big(\sE_3 \cup \sE_4 \cup \sE_5 \big)$, and hence it follows from the definition in~\eqref{eq:Def_ErrorExp} that $E(\Psi_{\SGA},\tilde{P}) = \min\{e_3,e_4,e_5\}$.

	\item We now prove the claim given by Proposition~\ref{prop:SGA_ErrExp}(b) where $P$ corresponds to a star structured tree. When the $4$ nodes form a star structure, then an error is made if only if the procedure in Algorithm~\ref{alg:Mod_StarNonStar} make any one of the following incorrect declarations:
\begin{itemize}
	\item Non-star with pair $\{X_1,X_2\}$.
	\item Non-star with pair $\{X_1,X_3\}$.
	\item Non-star with pair $\{X_1,X_4\}$.
\end{itemize}
By symmetry of the underlying star structure, the probability of each of the three erroneous declarations is same, and hence it follows from Algorithm~\ref{alg:Mod_StarNonStar} that $E(\Psi_{\SGA},\tilde{P})$ is equal to the exponent of the probability that $v_2 = \frac{\sqrt{|\widehat{\rho}_{1,3}\,\widehat{\rho}_{2,4}\, \widehat{\rho}_{1,4}\,\widehat{\rho}_{2,3}|}}{|\widehat{\rho}_{1,2}\,\widehat{\rho}_{3,4}|} < \alpha$. Finally, using Sanov's theorem~\citep[Chap.~11]{CoverBook06}, it follows that the exponent of the probability that this condition is satisfied is given by~\eqref{eq:SGA_ErrExpStar}.

\end{itemize}
\qed

\section{Simulation results for $4$-node homogeneous trees} \label{app:4nodeHomogeneousTrees}

Sec.~\ref{sec:ErrExpNumerical} presented numerical results comparing the error exponents using the $\Psi_{\SGA}$ and $\Psi_{\KA}$ algorithms, derived using the \emph{large deviation theory}~\citep[Sec.~11.4]{CoverBook06},  for $4$-node homogeneous trees. In this appendix, we present Monte Carlo simulation results for $4$-node homogeneous trees, that corroborate the results in Sec.~\ref{sec:ErrExpNumerical}. For any given tree structure, and any value of $n$ (number of samples), the error probability using $\Psi_{\SGA}$ or $\Psi_{\KA}$ is computed based on $10^5$ iterations (or runs) in the simulation setup.

\begin{figure}[t]
	\centering
	\includegraphics[width=0.50\textwidth, angle=0]{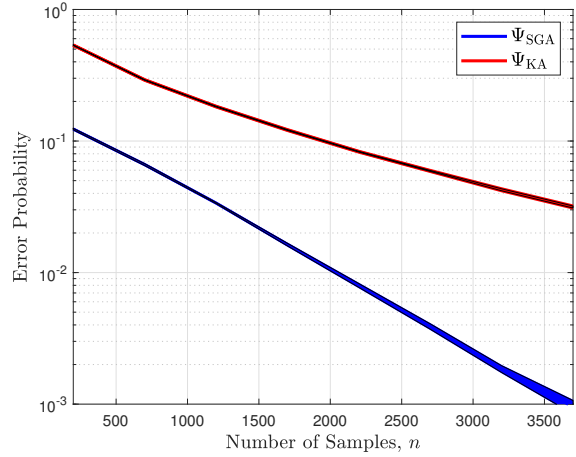}
	\caption{Comparison of error probabilities for $4$-node noiseless chains when all the edge correlation are equal to $\rho = 0.4$.}
	\label{Fig:ErrProb_4Chain_Noiseless_tPoint3}
\end{figure}

Fig.~\ref{Fig:ErrProb_4Chain_Noiseless_tPoint3} compares the error probabilities for $4$ node noiseless chains, using $\Psi_{\SGA}$ and $\Psi_{\KA}$, when all the edge correlation are equal to $\rho = 0.4$. Here, we consider all $12$ distinct chain structures using $4$ nodes, and for any given $n$ we compute the empirical mean, denoted $\mu_n$, and the empirical standard deviation, denoted $\sigma_n$, of the error probabilities using the $12$ distinct chain structures. The shaded area in blue (resp.\ red) in Fig.~\ref{Fig:ErrProb_4Chain_Noiseless_tPoint3} corresponds to the region between $\mu_n(\Psi_{\SGA}) + \sigma_n(\Psi_{\SGA})$ and $\mu_n(\Psi_{\SGA}) - \sigma_n(\Psi_{\SGA})$ (resp.\ between $\mu_n(\Psi_{\KA}) + \sigma_n(\Psi_{\KA})$ and $\mu_n(\Psi_{\KA}) - \sigma_n(\Psi_{\KA})$). The negative slope of the error probability curve is indicative of the error exponent, and Fig.~\ref{Fig:ErrProb_4Chain_Noiseless_tPoint3} demonstrates that the error exponent using $\Psi_{\SGA}$ is much higher than that using $\Psi_{\KA}$ when $\rho = 0.4$, as shown by the corresponding error exponent values in Fig.~\ref{Fig:ErrExp_Chain}(a).

\begin{figure}[t]
	\centering
	\includegraphics[width=0.50\textwidth, angle=0]{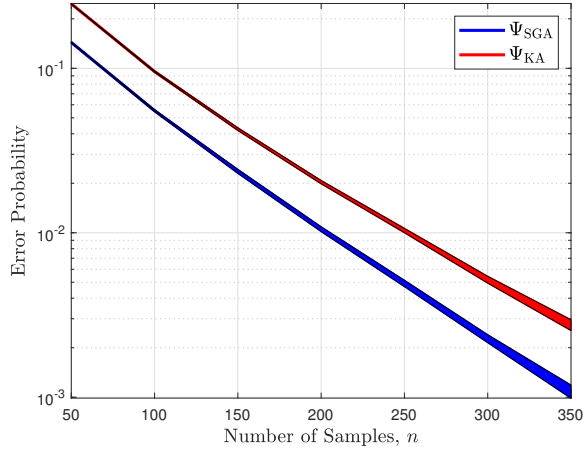}
	\caption{Comparison of error probabilities for $4$-node noiseless chains when all the edge correlation are equal to $\rho = 0.8$.}
	\label{Fig:ErrProb_4Chain_Noiseless_tPoint1}
\end{figure}

Fig.~\ref{Fig:ErrProb_4Chain_Noiseless_tPoint1} compares the error probabilities for $4$ node noiseless chains, using $\Psi_{\SGA}$ and $\Psi_{\KA}$, when all the edge correlation are equal to $\rho = 0.8$. Again, the shaded area in blue (resp.\ red) in Fig.~\ref{Fig:ErrProb_4Chain_Noiseless_tPoint1} corresponds to the region between $\mu_n(\Psi_{\SGA}) + \sigma_n(\Psi_{\SGA})$ and $\mu_n(\Psi_{\SGA}) - \sigma_n(\Psi_{\SGA})$ (resp.\ between $\mu_n(\Psi_{\KA}) + \sigma_n(\Psi_{\KA})$ and $\mu_n(\Psi_{\KA}) - \sigma_n(\Psi_{\KA})$), where $\mu_n$ denotes the empirical mean and $\sigma_n$ denotes the empirical standard deviation, for the error probabilities obtained using the $12$ distinct chain structures. Fig.~\ref{Fig:ErrProb_4Chain_Noiseless_tPoint1} shows that the slope of the error probability curves for $\Psi_{\SGA}$ and $\Psi_{\KA}$ are roughly equal when $\rho = 0.8$, as indicated by the corresponding error exponent values in Fig.~\ref{Fig:ErrExp_Chain}(a).

\begin{figure}[t]
	\centering
	\includegraphics[width=0.50\textwidth, angle=0]{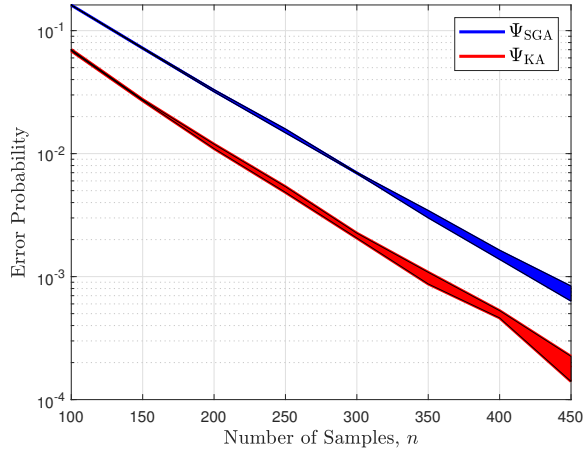}
	\caption{Comparison of error probabilities for $4$-node star structured trees when all the edge correlation are equal to $\rho = 0.6$ and $q_{\max}=0$.}
	\label{Fig:ErrProb_4Star_Noiseless_tPoint2}
\end{figure}

Fig.~\ref{Fig:ErrProb_4Star_Noiseless_tPoint2} compares the error probabilities for $4$ node star structured trees, using $\Psi_{\SGA}$ and $\Psi_{\KA}$. The shaded area in blue (resp.\   red) in Fig.~\ref{Fig:ErrProb_4Star_Noiseless_tPoint2} corresponds to the region between $\mu_n(\Psi_{\SGA}) + \sigma_n(\Psi_{\SGA})$ and $\mu_n(\Psi_{\SGA}) - \sigma_n(\Psi_{\SGA})$ (resp. \  between $\mu_n(\Psi_{\KA}) + \sigma_n(\Psi_{\KA})$ and $\mu_n(\Psi_{\KA}) - \sigma_n(\Psi_{\KA})$), where $\mu_n$ denotes the empirical mean and $\sigma_n$ denotes the empirical standard deviation, for the error probabilities obtained using the $4$ distinct star structured trees with $4$ nodes. Fig.~\ref{Fig:ErrProb_4Star_Noiseless_tPoint2} shows that the slopes of the error probability curves for $\Psi_{\SGA}$ and $\Psi_{\KA}$ are not very different when $\rho = 0.6$, as suggested by the error exponent values in Fig.~\ref{Fig:ErrExp_Star}(a).

Overall, the simulation results in Fig.~\ref{Fig:ErrProb_4Chain_Noiseless_tPoint3} and Fig.~\ref{Fig:ErrProb_4Chain_Noiseless_tPoint1} show that the error probability values do not deviate much due to the specific choice of a $4$-node chain structure, while Fig.~\ref{Fig:ErrProb_4Star_Noiseless_tPoint2} indicates a similar behavior for $4$-node star structured trees.

\section{Different $12$-node tree structures considered in Section~\ref{sec:NumericalResults}} \label{app:3TreeStructures}

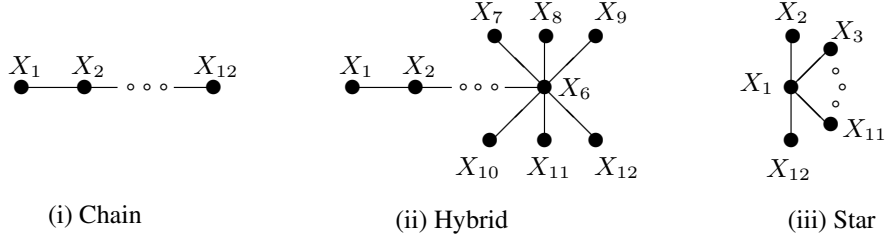
\begin{figure*}[t]
	{	\centering
		\scalebox{1.1}{
			\setlength{\unitlength}{.2mm} 	
			\begin{picture}(500,141)(20,-36)				
				\put(20,52){\circle*{8}}
				\put(23,52){\line(1,0){36}}
				\put(12,60){{\small $X_1$}}
				\put(58,52){\circle*{8}}
				\put(60,52){\line(1,0){18}}
				\put(50,60){{\small $X_2$}}
				\put(86,52){\circle{4}}
				\put(96,52){\circle{4}}
				\put(106,52){\circle{4}}
				\put(112,52){\line(1,0){21}}
				\put(136,52){\circle*{8}}
				\put(124,60){{\small $X_{12}$}}
				\put(36,-30){{\small (i) Chain}}

				\put(220,52){\circle*{8}}
				\put(223,52){\line(1,0){36}}
				\put(212,60){{\small $X_1$}}
				\put(258,52){\circle*{8}}
				\put(260,52){\line(1,0){20}}
				\put(250,60){{\small $X_2$}}
				\put(287,52){\circle{4}}
				\put(296,52){\circle{4}}
				\put(306,52){\circle{4}}
				\put(312,52){\line(1,0){21}}
				\put(336,52){\circle*{8}}
				\put(344,47){{\small $X_{6}$}}
				\put(336,52){\line(-1,1){30}}
				\put(306,83){\circle*{8}}				
				\put(291,93){{\small $X_{7}$}}
				\put(336,52){\line(0,1){30}}
				\put(337,83){\circle*{8}}				
				\put(327,93){{\small $X_{8}$}}
				\put(336,52){\line(1,1){31}}
				\put(367,83){\circle*{8}}				
				\put(366,93){{\small $X_{9}$}}
				\put(336,52){\line(-1,-1){31}}
				\put(303,20){\circle*{8}}				
				\put(282,-1){{\small $X_{10}$}}
				\put(336,52){\line(0,-1){31}}
				\put(336,20){\circle*{8}}				
				\put(325,-1){{\small $X_{11}$}}
				\put(336,52){\line(1,-1){31}}
				\put(367,20){\circle*{8}}				
				\put(366,-1){{\small $X_{12}$}}
				\put(246,-33){{\small (ii) Hybrid}}

				\put(485,52){\circle*{8}}				
				\put(454,50){{\small $X_1$}}
				\put(485,52){\line(0,1){31}}
				\put(486,83){\circle*{8}}				
				\put(475,93){{\small $X_2$}}
				\put(485,52){\line(1,1){22}}
				\put(509,75){\circle*{8}}				
				\put(509,81){{\small $X_3$}}
				\put(512,62){\circle{4}}				
				\put(516,52){\circle{4}}				
				\put(512,42){\circle{4}}				
				\put(485,52){\line(1,-1){22}}
				\put(509,30){\circle*{8}}				
				\put(516,22){{\small $X_{11}$}}
				\put(485,52){\line(0,-1){31}}
				\put(485,20){\circle*{8}}				
				\put(470,-2){{\small $X_{12}$}}
				\put(483,-33){{\small (iii) Star}}
			\end{picture}
		}
		\caption{Three different $12$-node tree structures}
		\label{Fig:3TreeStructures}
	}
\end{figure*}
Fig.~\ref{Fig:3TreeStructures} presents $12$-node trees with three different tree structures: (i) Chain, (ii) Hybrid, (iii) Star. The chain and the star structures are known to be extremal tree structures in terms of the error probability~\cite{Tan10TSP,TandonTZ20_JSAIT}, while the hybrid tree structure is a combination of the chain and star structures.

\section{Extension of \citet{Katiyar20arxiv} and SGA to Gaussian trees with numerical results} \label{app:gauss}
 We compare the error probabilities of $\Psi_{\SGA}$ and $\Psi_{\KA}$ in recovering the trees from tree-structured Gaussian graphical models with $d=10$ nodes. The tree structures used for comparison are similar to those in the Ising model experiments in Sec.~\ref{sec:NumericalResults}. Each observation of $X_i \in \mathbb{R}$ is corrupted by independent but non-identically distributed Gaussian noise such that the observed variable $Y_i= X_i + N_i$, where the noise $N_i\sim \mathcal{N}(0,\sigma^2_i)$ for some $\sigma_i >0$.
\subsection{Experiment setup}
Generating samples for Gaussian models begins with choosing a tree structure $\T_P  = (\mathcal{V}, \mathcal{E}_P)$ where the number of nodes $d=10$. We then generate the inverse covariance matrix $(\mathbf{\Sigma}^*)^{-1}$, by setting  its $(i,j)^{\text{th}}$ entry as
\begin{equation} \label{eq:Def_w}
  [(\mathbf{\Sigma}^*)^{-1}]_{i,j} =
  \begin{cases}
    w, & \text{if $(i,j)\in \mathcal{E}_P$;} \\
    1, & \text{if $i = j$;}\\
    0 & \text{otherwise} \\
  \end{cases}
\end{equation}
for some parameter $w\in\mathbb{R}$. This matrix is then inverted to get the covariance matrix $\mathbf{\Sigma}^*$ of the distribution $P$.  The correlation matrix $\textbf{K}^*$ is calculated from $\mathbf{\Sigma}^*$ using the formula $\textbf{K}^*= (\text{diag}(\mathbf{\Sigma}^*))^{-\frac{1}{2}}\mathbf{\Sigma}^*(\text{diag}(\mathbf{\Sigma}^*))^{-\frac{1}{2}}$. This is used to compute the minimum and maximum correlation coefficients $\rho_{\min}$ and $\rho_{\max}$. The parameter $w\in \mathbb{R}$ is   chosen so as to ensure that $\rho_{\max} \approx 0.8$ from the resulting $\textbf{K}^*$. 
Additionally, for the noiseless case, the diagonal matrix $\textbf{D}^*$ is taken to be the zero matrix, and for the noisy case $[\textbf{D}^*]_{i,i}=2$ for $i \in \{1,3,5,7,9\}$; thus Gaussian noise of variance $2$ is directly added to the node observations for nodes with odd indices. Finally, samples were generated from the joint Gaussian distribution $\tilde{P}(\mathbf{y})=\mathcal{N}(\mathbf{y};\mathbf{0}, \mathbf{\Sigma}^*+\textbf{D}^*)$.

\subsection{Modifications to the Algorithm}
Even though \citet{Katiyar19ICML} proposed an algorithm for the partial learning of Gaussian graphical models (given noisy observations), it is not directly implementable to the case in which we have a {\em finite} number of samples $n$. An algorithm for learning trees with noisy samples up to their equivalence classes was proposed by \citet{Katiyar20arxiv} but the algorithm provided therein was originally used to recover {\em Ising  models}. Hence, some modifications to the algorithm had to be made so that it is amenable to  Gaussian graphical models.  In particular, there are two thresholds given to form the proximal set of any node $i$ in the Ising case. The first, $t_1 \triangleq (1-2q_{\max})^2\rho^4_{\min}$, gives a lower bound for the correlation between $i$ any other node with distance at most 4 between them. The second, $t_2 \triangleq \min\big\{t_1, \frac{t_1(1-2q_{\max})}{\rho_{\max}} \big\}$, gives another lower bound for the correlation of $i$ and the first node in the path $i$ to some other node $j$ where $\tilde{\rho}_{i,j}\geq t_1$.

A similar idea can be used to construct proximal sets for the Gaussian case using the correlation decay property. First, note that the correlation  coefficient between the variable $X_i$ and its noisy counterpart $Y_i=X_i+N_i$  is
\begin{align}
    \rho_{X_iY_i} = \frac{\mathbb{E}[X_i Y_i]}{\sqrt{\mathbb{E}[X^2_i]\mathbb{E}[ Y^2_i]}} &=  \frac{\mathbb{E}[X_i(X_i+N_i)]}{\sqrt{\mathbb{E}[X^2_i] \mathbb{E}[(X_i+N_i)^2]}}=\frac{\mathbb{E}[X_i^2 ]}{\sqrt{\mathbb{E}[X^2_i] (\mathbb{E}[ X_i^2] +\mathbb{E}[N_i^2] ) }}= \frac{1}{\sqrt{1 + \frac{\sigma^2_i}{\mathbb{E}[X^2_i]}}}. \label{eqn:defS}
\end{align}
Let $S_i \triangleq {\sigma^2_i}/{\mathbb{E}[X^2_i]}$ and $S_{\max} \triangleq \max_{1\leq i \leq d} S_i$. With the correlation decay property for Gaussian tree models~\citep[Eqn.~(18)]{Tan10TSP}, any node $j$ within radius $4$ of $i$ will have its noisy correlation coefficient bounded as follows
\begin{align}
    \rho_{Y_iY_j} &= \rho_{Y_iX_i}\cdot \rho_{X_iX_j}\cdot \rho_{X_jY_j}
    \geq   \frac{1}{\sqrt{1 + \frac{\sigma^2_i}{\mathbb{E}[X^2_i]}}}\cdot\rho_{\min}^4\cdot \frac{1}{\sqrt{1 + \frac{\sigma^2_j}{\mathbb{E}[X^2_j]}}}
    \geq \rho_{\min}^4\cdot \frac{1}{1+S_{\max}}.
\end{align}
In this way, we obtain the first threshold $h_1 \triangleq \frac{\rho_{\min}^4}{1+S_{\max}}$. Similarly, if we let node $k$ be the first node in the path from $j$ to $i$ (i.e., $k$ is at distance $1$ from $j$), then it holds that 
\begin{equation}
\frac{\rho_{Y_i Y_j}}{\rho_{X_k  X_j}} = \underbrace{ \rho_{X_i Y_i}\cdot\rho_{X_i X_k}\cdot\rho_{X_k Y_k}}_{\rho_{Y_iY_k}}\cdot\frac{\rho_{X_jY_j}}{\rho_{X_k Y_k}},
\end{equation}
whence, by \eqref{eqn:defS}  and the definition of $S_i$, the correlation coefficient  between $Y_i$ and $Y_k$ can be bounded as follows
\begin{align}
    \rho_{Y_iY_k} = \frac{\rho_{Y_iY_j}}{\rho_{X_kX_j}}\cdot \sqrt{ \frac{1+S_j}{1+S_k} }
    \geq \frac{h_1}{\rho_{\max}\sqrt{1+S_{\max}}}.
\end{align}
Thus, we get our second threshold $h_2 \triangleq \min\big\{h_1, \frac{h_1}{\rho_{\max} \sqrt{1+S_{\max}}}\big\}$. So the proximal set of node $i$ is defined as the set of all nodes $j$ that satisfy $|\hat{\rho}_{i,j}|\geq 0.5h_2$.

\subsection{Numerical Results}\label{sec:gauss_num}
We now present numerical results for Gaussian tree models, comparing the performance of $\Psi_{\SGA}$ and $\Psi_{\KA}$ for three different tree structures with $d=10$ nodes: (i)~Chain, (ii)~Hybrid, and (iii)~Star. For a given tree structure $\T$, and $n$ noisy samples $\brY_1^n$, the error probability $\bbP\big(\Psi(\brY_1^n) \notin [\T] \big)$ for a given learning algorithm $\Psi$, is estimated using $10^5$ iterations (or runs) in the Monte Carlo simulation framework, where an error is declared if the estimated tree does not belong to the equivalence class $[\T]$. For the noisy case, we set $[\textbf{D}^*]_{i,i}=2$ for $i \in \{1,3,5,7,9\}$.

\begin{figure}[t]
	\centering
	\includegraphics[width=0.60\textwidth, angle=0]{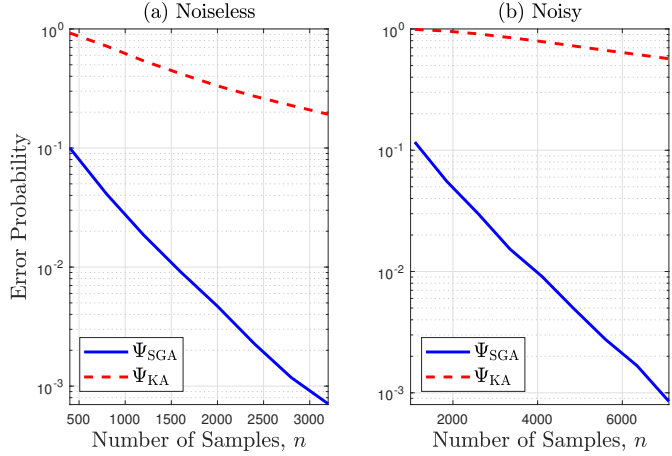}
	\caption{Comparison of error probabilities for a $10$-node chain, with $w=0.5$ (see~\eqref{eq:Def_w}).}
	\label{Fig:Gaussian_ErrProb_10Chain_wPoint5}
\end{figure}

\subsubsection{$10$-node chain}

Fig.~\ref{Fig:Gaussian_ErrProb_10Chain_wPoint5} plots the results for a $10$-node chain for the (a)~noiseless and (b)~noisy cases, with $w=0.5$ (see~\eqref{eq:Def_w}). It is seen that $\Psi_{\SGA}$ significantly outperforms $\Psi_{\KA}$ for the Gaussian chain; a similar trend was observed for the Ising chain in Fig.~\ref{Fig:ErrProb_12Chain}.

\begin{figure}[t]
	\centering
	\includegraphics[width=0.60\textwidth, angle=0]{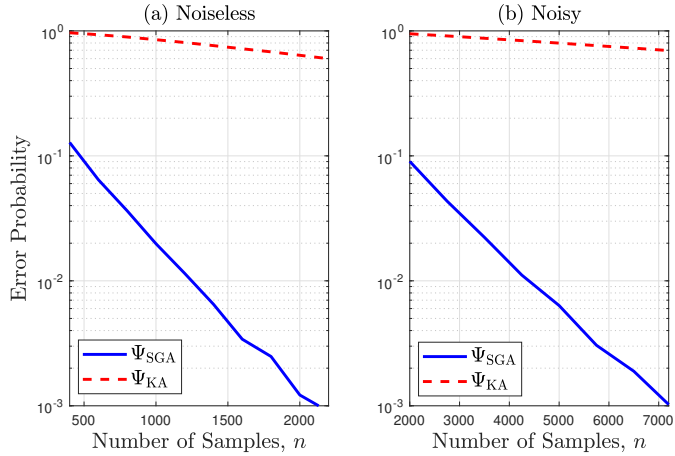}
	\caption{Comparison of error probabilities for a $10$-node hybrid tree structure, with $w=0.38$.}
	\label{Fig:Gaussian_ErrProb_10Hybrid_wPoint38}
\end{figure}

\subsubsection{$10$-node hybrid tree structure}

Fig.~\ref{Fig:Gaussian_ErrProb_10Hybrid_wPoint38} plots the results for a $10$-node hybrid tree structure for the (a)~noiseless and (b)~noisy cases, with $w=0.38$ (see~\eqref{eq:Def_w}). The hybrid tree structure is a combination of chain and star structures where nodes $1$ to $5$ are linked in the form of a chain, while nodes $6$ to $10$ are directly connected to node $5$. Similar to the performance comparison for the Ising hybrid tree in Fig.~\ref{Fig:ErrProb_12Hybrid}, it is seen that $\Psi_{\SGA}$ significantly outperforms $\Psi_{\KA}$ for the Gaussian hybrid tree.

\subsubsection{$10$-node star}

\begin{figure}[t]
	\centering
	\includegraphics[width=0.60\textwidth, angle=0]{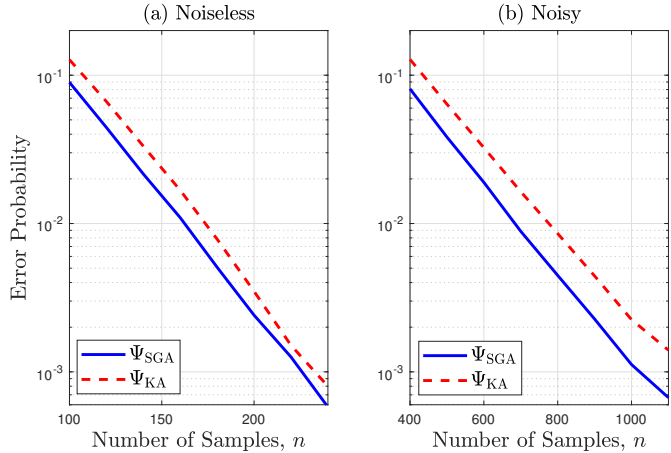}
	\caption{Comparison of error probabilities for a $10$-node star tree, with $w=0.325$.}
	\label{Fig:Gaussian_ErrProb_10Star_wPoint325}
\end{figure}

Fig.~\ref{Fig:Gaussian_ErrProb_10Star_wPoint325} plots the results for a $10$-node star tree structure for the (a)~noiseless and (b)~noisy cases, with $w=0.325$ (see~\eqref{eq:Def_w}), where nodes $2$ to $10$ are directly connected to node $1$. It is seen that the performance of $\Psi_{\SGA}$ is only slightly better than that of $\Psi_{\KA}$; a similar trend was observed for the Ising star tree in Fig.~\ref{Fig:ErrProb_12Star_tPoint2}. This is corroborated by the error exponent results in Fig.~\ref{Fig:ErrExp_Star}, which is for Ising models but we expect the same behavior for Gaussian models. 

These experiments for Gaussian tree models learned using noisy samples demonstrate that the behavior of SGA and the algorithm by \citet{Katiyar20arxiv} are qualitatively very similar to the Ising case as detailed in Sec.~\ref{sec:NumericalResults}. In particular, the performance of SGA is far superior to that of \citet{Katiyar20arxiv} for chains and hybrid trees (i.e., trees with moderate to large diameter). SGA's performance is comparable to the algorithm by \citet{Katiyar20arxiv} for stars (i.e., trees with small diameter), though in this case, SGA's performance is still marginally better. 

\end{document}